\documentclass[11pt,letterpaper]{article}
\usepackage[T1]{fontenc}
\usepackage[utf8]{inputenc}

\usepackage{amsmath,amssymb,amsthm,mathtools}
\usepackage{newpxtext}                 
\usepackage{newpxmath}                 
\usepackage[scaled=0.90]{FiraSans}     

\usepackage[letterspace=80]{microtype}

\usepackage[letterpaper,left=1.15in,right=1.15in,top=1.1in,bottom=1.15in,headsep=0.32in,footskip=0.45in]{geometry}
\usepackage{xcolor}
\definecolor{ink}{HTML}{1A1F2B}
\definecolor{navy}{HTML}{1F3A63}         
\definecolor{refcool}{HTML}{2C5AA0}      
\definecolor{citewarm}{HTML}{B5403B}     
\definecolor{muted}{HTML}{6B6A66}        
\definecolor{rulegray}{HTML}{D3D6DC}
\definecolor{thmback}{HTML}{F2F6FC}
\definecolor{lemback}{HTML}{F5F6F8}
\definecolor{lembar}{HTML}{6B7A90}
\definecolor{absback}{HTML}{F7F8FA}

\usepackage{graphicx,booktabs,array,tabularx,multicol}
\usepackage{enumitem}
\setlist{itemsep=1pt,parsep=0pt,topsep=3pt,partopsep=0pt}
\usepackage[font=small,labelfont={sf,bf,color=navy},labelsep=period,skip=6pt]{caption}
\usepackage[round]{natbib}

\usepackage[most]{tcolorbox}
\usepackage{titlesec,titletoc}
\usepackage{fancyhdr}
\usepackage{needspace}
\usepackage{algorithm}
\usepackage{tikz}
\usepackage{aliascnt}
\usetikzlibrary{positioning,calc,arrows.meta,backgrounds,fit}
\usepackage{needspace}

\titleformat{\section}{\sffamily\Large\bfseries\color{navy}}{\textcolor{refcool}{\thesection}}{0.7em}{}
\titlespacing*{\section}{0pt}{3.2ex plus 1ex minus .2ex}{1.5ex plus .2ex}
\titleformat{\subsection}{\sffamily\large\bfseries\color{navy}}{\textcolor{refcool}{\thesubsection}}{0.6em}{}
\titlespacing*{\subsection}{0pt}{2.4ex plus .8ex minus .2ex}{1.0ex plus .2ex}
\titleformat{\subsubsection}{\sffamily\normalsize\bfseries\color{navy}}{\textcolor{refcool}{\thesubsubsection}}{0.5em}{}
\titleformat{\paragraph}[runin]{\sffamily\bfseries\color{navy}}{}{0pt}{}
\titlespacing*{\paragraph}{0pt}{1.4ex plus .5ex minus .2ex}{0.55em}

\titlecontents{section}[2.0em]{\addvspace{0.6pt}\sffamily\fontsize{8.4}{10.2}\selectfont}
  {\contentslabel[\textcolor{refcool}{\bfseries\thecontentslabel}]{2.0em}}{\hspace*{-2.0em}}
  {\titlerule*[0.6pc]{\textcolor{rulegray}{.}}\contentspage}

\fancypagestyle{first}{\fancyhf{}\fancyfoot[C]{\sffamily\footnotesize\color{muted}\thepage}}

\newtheoremstyle{boxed}{0pt}{0pt}{\itshape}{}{\sffamily\bfseries\color{navy}}{.}{0.5em}
  {\thmname{#1}\thmnumber{ #2}\thmnote{\ \textmd{(#3)}}}
\newtheoremstyle{boxedrm}{0pt}{0pt}{\normalfont}{}{\sffamily\bfseries\color{navy}}{.}{0.5em}
  {\thmname{#1}\thmnumber{ #2}\thmnote{\ \textmd{(#3)}}}
\usepackage{aliascnt}
\theoremstyle{boxed}
\newtheorem{theorem}{Theorem}[section]
\newaliascnt{proposition}{theorem}\newtheorem{proposition}[proposition]{Proposition}\aliascntresetthe{proposition}
\newaliascnt{lemma}{theorem}\newtheorem{lemma}[lemma]{Lemma}\aliascntresetthe{lemma}
\newaliascnt{corollary}{theorem}\newtheorem{corollary}[corollary]{Corollary}\aliascntresetthe{corollary}
\newaliascnt{claim}{theorem}\aliascntresetthe{claim}
\newaliascnt{conjecture}{theorem}\aliascntresetthe{conjecture}
\newaliascnt{fact}{theorem}\aliascntresetthe{fact}
\newaliascnt{observation}{theorem}\aliascntresetthe{observation}
\theoremstyle{boxedrm}
\newaliascnt{definition}{theorem}\newtheorem{definition}[definition]{Definition}\aliascntresetthe{definition}
\newaliascnt{assumption}{theorem}\aliascntresetthe{assumption}
\newaliascnt{remark}{theorem}\newtheorem{remark}[remark]{Remark}\aliascntresetthe{remark}
\newaliascnt{example}{theorem}\newtheorem{example}[example]{Example}\aliascntresetthe{example}
\newaliascnt{problem}{theorem}\aliascntresetthe{problem}
\newaliascnt{question}{theorem}\aliascntresetthe{question}
\tcbset{thmbox/.style={enhanced,breakable,frame hidden,sharp corners,boxrule=0pt,
  left=8pt,right=8pt,top=5pt,bottom=5pt,before skip=9pt,after skip=9pt}}
\tcbset{coolbox/.style={thmbox,colback=thmback,borderline west={2.4pt}{0pt}{refcool}}}
\tcbset{slatebox/.style={thmbox,colback=lemback,borderline west={2.4pt}{0pt}{lembar}}}
\tcbset{warmbox/.style={thmbox,colback=white,borderline west={2.4pt}{0pt}{citewarm}}}
\tcbset{linebox/.style={blanker,breakable,left=8pt,before skip=7pt,after skip=7pt,borderline west={1pt}{0pt}{rulegray}}}
\tcolorboxenvironment{theorem}{coolbox}
\tcolorboxenvironment{proposition}{coolbox}
\tcolorboxenvironment{corollary}{coolbox}
\tcolorboxenvironment{conjecture}{coolbox}
\tcolorboxenvironment{lemma}{slatebox}
\tcolorboxenvironment{claim}{slatebox}
\tcolorboxenvironment{fact}{slatebox}
\tcolorboxenvironment{observation}{slatebox}
\tcolorboxenvironment{definition}{warmbox}
\tcolorboxenvironment{assumption}{warmbox}
\tcolorboxenvironment{remark}{linebox}
\tcolorboxenvironment{example}{linebox}
\tcolorboxenvironment{problem}{linebox}
\tcolorboxenvironment{question}{linebox}

\makeatletter
\renewenvironment{proof}[1][\proofname]{\par\pushQED{\qed}\normalfont\topsep6\p@\@plus6\p@\relax
  \trivlist\item[\hskip\labelsep{\sffamily\bfseries\color{navy}#1.}]\ignorespaces}{\popQED\endtrivlist\@endpefalse}
\makeatother

\makeatletter

\newcommand{\X}{\mathcal X}
\newcommand{\Hyp}{\mathcal H}
\newcommand{\F}{\mathcal F}
\newcommand{\Nat}{\mathbb N}
\newcommand{\one}{\mathbf 1}
\newcommand{\supp}{\operatorname{supp}}
\newcommand{\Txt}{\mathsf{Txt}}
\newcommand{\Inff}{\mathsf{Inf}}
\newcommand{\Ctr}{\mathsf{Ctr}}
\newcommand{\Id}{\mathrm{Id}}
\newcommand{\Gen}{\mathrm{Gen}}
\newcommand{\Fin}{\mathrm{Fin}}
\newcommand{\kCtrId}[1]{#1\text{-}\Ctr\Id}
\newcommand{\kTxtId}[1]{#1\text{-}\Txt\Id}
\newcommand{\Seen}{\mathrm{Seen}}
\newcommand{\Safe}{\mathrm{Safe}}
\newcommand{\Cut}{\Delta}
\newcommand{\GammaC}{\Gamma}
\newcommand{\V}{\operatorname{V}}
\newcommand{\anti}[1]{\overline{#1}}
\newcommand{\Defect}{\operatorname{Def}}
\newcommand{\kb}{\bar\kappa}
\newcommand{\viol}{\operatorname{viol}}
\newcommand{\HDelta}{\Hyp_{\Cut}}
\newcommand{\CDelta}{\mathrm C_{\Cut}}
\newcommand{\Hco}{\Hyp_{\mathrm{co}}}
\makeatother

\usepackage[colorlinks=true,linkcolor=refcool,citecolor=citewarm,urlcolor=refcool,
  pdftitle={Contrastive Identification and Generation in the Limit},pdfauthor={Xiaoyu Li, Andi Han, Jiaojiao Jiang, Junbin Gao},pdfsubject={},pdfkeywords={identification in the limit, generation in the limit, contrastive presentations, tell-tale sets, adversarial corruption, learning theory}]{hyperref}
\usepackage[capitalize,noabbrev,nameinlink]{cleveref}
\crefname{steps}{step}{steps}
\crefname{algorithm}{Algorithm}{Algorithms}
\Crefname{algorithm}{Algorithm}{Algorithms}

\begin{document}
\thispagestyle{first}

\noindent{\color{refcool}\rule{\linewidth}{2.2pt}}\par\vspace{7pt}
\noindent{\sffamily\footnotesize\color{muted}\lsstyle NEURIPS 2026\quad\textbullet\quad EXTENDED VERSION\quad\textbullet\quad November 2026}\par\vspace{8pt}
\noindent{\sffamily\bfseries\fontsize{23}{28}\selectfont\color{navy}Contrastive Identification and Generation\\ in the Limit\par}\vspace{7pt}
\vspace{6pt}
\noindent\begin{tabular}{@{}>{\raggedright\arraybackslash}p{\dimexpr(\linewidth-10pt)/2\relax}@{\hspace{10pt}}>{\raggedright\arraybackslash}p{\dimexpr(\linewidth-10pt)/2\relax}@{}}
{\sffamily\bfseries\color{ink}Xiaoyu Li} & {\sffamily\bfseries\color{ink}Andi Han} \\[1pt]
{\small University of New South Wales} & {\small University of Sydney} \\[1pt]
{\sffamily\footnotesize\href{mailto:xiaoyu.li2@unsw.edu.au}{\color{muted}xiaoyu.li2@unsw.edu.au}} & {\sffamily\footnotesize\href{mailto:andi.han@sydney.edu.au}{\color{muted}andi.han@sydney.edu.au}}
\end{tabular}\par\vspace{6pt}\noindent\begin{tabular}{@{}>{\raggedright\arraybackslash}p{\dimexpr(\linewidth-10pt)/2\relax}@{\hspace{10pt}}>{\raggedright\arraybackslash}p{\dimexpr(\linewidth-10pt)/2\relax}@{}}
{\sffamily\bfseries\color{ink}Jiaojiao Jiang} & {\sffamily\bfseries\color{ink}Junbin Gao} \\[1pt]
{\small University of New South Wales} & {\small University of Sydney} \\[1pt]
{\sffamily\footnotesize\href{mailto:jiaojiao.jiang@unsw.edu.au}{\color{muted}jiaojiao.jiang@unsw.edu.au}} & {\sffamily\footnotesize\href{mailto:junbin.gao@sydney.edu.au}{\color{muted}junbin.gao@sydney.edu.au}}
\end{tabular}

\begingroup\renewcommand\thefootnote{}\footnotetext{A preliminary version of this paper was accepted at the 40th Conference on Neural Information Processing Systems (NeurIPS 2026).}\endgroup
\vspace{10pt}

\begin{tcolorbox}[enhanced,colback=absback,frame hidden,sharp corners,boxrule=0pt,
  borderline west={3pt}{0pt}{refcool},left=11pt,right=11pt,top=7pt,bottom=7pt]
{\sffamily\bfseries\footnotesize\color{navy}\lsstyle ABSTRACT}\par\vspace{3pt}
\small In identification in the limit \citep{gold67}, a learner reads positive examples and must eventually name the target hypothesis; in generation in the limit \citep{km24}, it must eventually produce new positive examples. We study a weaker, relational signal: each observation is an unordered pair $\{x,y\}$ with $h(x)\ne h(y)$, such as a cannot-link constraint or a dissimilar-pair label, and the learner never learns which point is positive. All our results come from one object: a pair is valid for every member of a finite family exactly when its points lie in antipodal cells of the family's membership map.
(1)~For every corruption budget $k$, known or unknown, we characterize exactly the identifiable classes. Incomparable hypotheses must cover the domain and share more than $2k$ points, and Angluin's tell-tale sets must survive the loss of $2k$ points. Every unbounded budget makes any two hypotheses confusable.
(2)~A contrastive closure dimension characterizes uniform generation. On finite classes, generation fails exactly when a family that shares a presentation has a finite common support; on countable classes this finite condition no longer suffices.
(3)~On clean data the four learning criteria form a strict diamond.
(4)~Under corruption the order changes: some classes are identifiable from contrastive pairs under every finite budget by one algorithm, but not from positive examples with one corrupted example.\par
\vspace{4pt}{\sffamily\footnotesize\color{navy}\textbf{Keywords:}}\ {\footnotesize identification in the limit, generation in the limit, contrastive presentations, tell-tale sets, adversarial corruption, learning theory}
\end{tcolorbox}

\enlargethispage{40pt}
\vspace{2pt}
\begin{tcolorbox}[enhanced,colback=white,colframe=rulegray,boxrule=0.5pt,arc=2pt,left=10pt,right=10pt,top=5pt,bottom=4pt]
{\sffamily\bfseries\footnotesize\color{navy}\lsstyle CONTENTS}\par
\setlength{\multicolsep}{4pt plus 1pt minus 1pt}\begin{multicols}{2}
{\hypersetup{linkcolor=ink}\makeatletter\@starttoc{toc}\makeatother}
\end{multicols}
\end{tcolorbox}
\suppressfloats[t]
\clearpage

\definecolor{hc}{HTML}{2C5AA0}
\definecolor{gc}{HTML}{B5403B}
\definecolor{tc}{HTML}{1B7F63}
\definecolor{inkc}{HTML}{1A1F2B}
\definecolor{grayc}{HTML}{6B6A66}
\definecolor{lightc}{HTML}{D3D6DC}
\definecolor{defc}{HTML}{C9821A}
\colorlet{acc}{hc}
\colorlet{hfill}{hc!13}
\colorlet{gfill}{gc!13}
\colorlet{bothfill}{hc!9!gc!11}
\colorlet{deffill}{defc!30}
\tikzset{
  figfont/.style={font=\small},
  pt/.style={circle, fill=inkc, inner sep=0pt, minimum size=4.2pt},
  ptpos/.style={circle, fill=hc, inner sep=0pt, minimum size=5pt},
  ptneg/.style={circle, draw=lightc, fill=white, line width=0.7pt, inner sep=0pt, minimum size=5pt},
  ptunk/.style={circle, draw=inkc, fill=white, line width=0.75pt, inner sep=0pt, minimum size=5pt},
  pairedge/.style={inkc!80, line width=0.8pt, line cap=round},
  crossedge/.style={inkc, line width=1.15pt, line cap=round},
  corruptedge/.style={defc, line width=1.0pt, dash pattern=on 3pt off 2pt, line cap=round},
  cutline/.style={grayc, line width=0.7pt, dash pattern=on 4pt off 2.5pt},
  hborder/.style={hc, line width=0.9pt},
  gborder/.style={gc, line width=0.9pt},
  frameX/.style={grayc!70, line width=0.5pt, rounded corners=3pt},
  cell/.style={font=\footnotesize, inner sep=1pt},
  paneltitle/.style={font=\small\sffamily\bfseries, color=inkc},
  panelnote/.style={font=\footnotesize, color=grayc, align=center},
  comp/.style={rounded corners=7pt, fill=lightc!40, draw=lightc, line width=0.5pt, inner sep=5pt},
  badge/.style={font=\scriptsize\sffamily, rounded corners=2pt, inner sep=1.6pt, draw=lightc, fill=white, text=inkc},
  arr/.style={-{Stealth[length=2.3mm, width=1.7mm]}, line width=0.75pt},
}

\section{Introduction}
\label{sec:introduction}

In identification in the limit \citep{gold67}, a learner knows a class $\Hyp$ of hypotheses and reads an infinite stream of examples of an unknown target $h\in\Hyp$. The learner must eventually output $h$ and never change its output again.
If the stream labels every example (an \emph{informant}), every countable class is identifiable.
If the stream shows only positive examples (a \emph{text}), simple classes are not.
The tell-tale theorem of \citet{angluin80} gives the exact condition for text identification: each hypothesis needs a finite set of positives that no proper sub-hypothesis in the class contains.
This condition has organized inductive inference for four decades \citep{lzz08}.

\citet{km24} introduced generation in the limit. Here the learner does not name the target. It must eventually output positives of the target that it has not seen. On countable classes with infinite supports this is always possible from a text, in sharp contrast to identification.
\citet{lrt25} restated the model in the notation of learning theory and introduced a \emph{closure dimension} that characterizes \emph{uniform} generation. In uniform generation, one bound on the number of examples must work for every target.
Many refinements followed; \cref{sec:related} reviews them.

All these models observe one positive example at a time.
Some supervision is relational instead, and it says only that two examples have \emph{different} labels.
A cannot-link constraint in constrained clustering requires two items to lie in different clusters \citep{wcrs01}. With two clusters, exactly one endpoint lies in the target cluster, and the constraint does not say which.
A dissimilar pair in pairwise weak supervision has the same form \citep{dbs21}; practitioners collect such pair labels when individual labels are sensitive or costly \citep{bns18,dbs21}.
A comparison that reveals a direction, such as ``$x$ is better than $y$'', is different: it orients the pair, and we discuss it in \cref{sec:related}.
This paper asks a basic question: what can a learner achieve in the limit if each observation says only that two examples disagree under the target?

\paragraph{The model.}
The target is a binary hypothesis $h:\X\to\{0,1\}$ with support $\supp(h):=\{x\in\X:h(x)=1\}$.
In each round the learner sees an unordered pair $\{x,y\}$ with $h(x)\ne h(y)$, and it does not learn which point is positive.\footnote{We use ``contrastive'' in the sense of pairwise similar/dissimilar supervision \citep{bns18,dbs21}, not in the sense of self-supervised representation learning.}
Every positive of $h$ must appear in some pair; the negatives need not all appear.
We call such a stream a \emph{contrastive presentation} of $h$.
There is a simple geometric reading.
Take $\X$ as the vertex set of a graph and each observed pair as an edge.
Then every edge crosses the unknown cut $(\supp(h),\X\setminus\supp(h))$, some edge touches every vertex on the positive side, and no vertex carries a mark (\cref{fig:teaser}).
As in the work of Gold, the adversarial and covering stream is an idealization: we study the limit behaviour of the pair signal, not a deployed pipeline.

\begin{figure}[t]
\centering
\begin{tikzpicture}[figfont, x=1.22cm, y=1.0cm,
  ptpos/.style={circle, fill=hc, inner sep=0pt, minimum size=6pt},
  ptneg/.style={circle, draw=grayc!55, fill=white, line width=0.8pt, inner sep=0pt, minimum size=6pt},
  ptunk/.style={circle, draw=inkc, fill=white, line width=0.85pt, inner sep=0pt, minimum size=6pt},
  pairedge/.style={inkc!80, line width=0.95pt, line cap=round}]
\def\pts{p1/0.30/1.00, p2/0.95/0.35, p3/0.20/-0.25, p4/0.95/-0.90, p5/0.30/-1.45}
\def\nts{n1/2.45/0.95, n2/2.10/0.05, n3/2.70/-0.70, n4/2.15/-1.40}
\begin{scope}
  \foreach \n/\x/\y [count=\i] in \pts {
    \node[ptpos] (a\n) at (\x,\y) {};
    \node[font=\scriptsize, text=hc, anchor=south west, inner sep=0.6pt] at (a\n.north east) {$x_{\i}$};
  }
  \foreach \n/\x/\y in \nts { \node[ptneg] at (\x,\y) {}; }
  \node[paneltitle] at (1.45,1.85) {(a) text};
  \node[panelnote] at (1.45,-2.35) {one positive per round};
\end{scope}
\begin{scope}[shift={(4.15,0)}]
  \foreach \n/\x/\y in \pts { \coordinate (b\n) at (\x,\y); }
  \foreach \n/\x/\y in \nts { \coordinate (b\n) at (\x,\y); }
  \draw[cutline] (1.55,1.40) .. controls (1.25,0.55) and (1.85,-0.55) .. (1.50,-1.75);
  \foreach \u/\v in {p1/n1, p2/n2, p3/n2, p4/n3, p5/n4} { \draw[pairedge] (b\u) -- (b\v); }
  \foreach \n/\x/\y in \pts { \node[ptunk] at (\x,\y) {}; }
  \foreach \n/\x/\y in \nts { \node[ptunk] at (\x,\y) {}; }
  \node[font=\scriptsize, text=grayc, anchor=west] at (1.66,1.38) {hidden cut};
  \node[paneltitle] at (1.45,1.85) {(b) contrastive presentation};
  \node[panelnote] at (1.45,-2.35) {one unordered pair per round};
\end{scope}
\begin{scope}[shift={(8.3,0)}]
  \foreach \n/\x/\y in \pts { \coordinate (c\n) at (\x,\y); }
  \foreach \n/\x/\y in \nts { \coordinate (c\n) at (\x,\y); }
  \begin{scope}[on background layer]
    \node[comp, fit=(cp1)(cn1)] (k1) {};
    \node[comp, fit=(cp2)(cp3)(cn2)] (k2) {};
    \node[comp, fit=(cp4)(cn3)] (k3) {};
    \node[comp, fit=(cp5)(cn4)] (k4) {};
  \end{scope}
  \foreach \u/\v in {p1/n1, p2/n2, p3/n2, p4/n3, p5/n4} { \draw[pairedge] (c\u) -- (c\v); }
  \foreach \n/\x/\y in \pts { \node[ptunk] at (\x,\y) {}; }
  \foreach \n/\x/\y in \nts { \node[ptunk] at (\x,\y) {}; }
  \foreach \k in {k1,k2,k3,k4} { \node[badge, anchor=west] at ([xshift=2pt]\k.east) {$\pm$}; }
  \node[paneltitle] at (1.45,1.85) {(c) what the learner can infer};
  \node[panelnote] at (1.45,-2.35) {labels up to one flip per component};
\end{scope}
\end{tikzpicture}
\caption{Three views of one target $h$; positives are on the left. (a)~A text shows positives. (b)~A contrastive presentation shows pairs that cross the unknown cut between $\supp(h)$ and its complement, and it never marks a point. (c)~Every observed pair crosses the cut, so on each connected component of the observed graph the labels are fixed up to one hidden flip (\cref{rem:hidden-bit}).}
\label{fig:teaser}
\end{figure}

\paragraph{Two tasks.}
We keep the two goals of the literature and change only the input.
An \emph{identifier} must eventually output $h$ and never change its output again.
A \emph{generator} must eventually output, in every round, a positive of $h$ that has not yet appeared in the stream.
So the generator reads unordered pairs but must still produce positive points.
This choice isolates the value of one bit per observation, the orientation of the pair.
It also makes a basic symmetry visible. The hypothesis $h$ and its complement $\one-h$ have the same valid pairs. So if a class contains both, no learner can separate them, and no generator can serve both, because their supports are disjoint.
Other members of the class can break this symmetry.
A generator can even certify a new positive before it knows the orientation of any pair. In the co-singleton class below, after the pair $\{s,z\}$ both $h_s$ and $h_z$ remain possible, but every other point is positive under both.

\paragraph{Two examples.}
The hidden direction can hurt, and it can help.
\begin{itemize}
  \item \emph{It hurts.} Let $h_A$ and $h_B$ have disjoint infinite supports $A=\{a_n\}$ and $B=\{b_n\}$. One positive example names the target. But the stream $\{a_1,b_1\},\{a_2,b_2\},\dots$ is a contrastive presentation of both hypotheses, so no learner can tell them apart.
  \item \emph{It helps.} Let $h_s$ label every example positive except one ``hole''~$s$. If a text may contain one wrong example, then the enumeration of $\X$ is a text for every $h_s$, and no learner can succeed. In a contrastive presentation of $h_s$, every correct pair contains $s$, so a corrupted pair is visible as a pair that misses the common point. A learner that counts, for each point, the pairs that miss it succeeds for every finite number of corrupted pairs, and it does not need to know this number.
\end{itemize}
Our results explain both phenomena with one piece of combinatorics.

\subsection{Main results}
\label{subsec:contributions}

We work with countable classes of proper nontrivial hypotheses, whose supports are neither empty nor all of $\X$, on a countably infinite example space.
We write $\Ctr\Id$ and $\Ctr\Gen$ for the classes that admit an identifier or a generator in the limit from contrastive presentations. The classes $\Txt\Id$ and $\Txt\Gen$ are the analogues for texts.
In $\kCtrId{k}$, the adversary may corrupt at most $k$ pairs of the stream and the learner knows~$k$. In $\kCtrId{\Fin}$, one learner must succeed for every finite number of corrupted pairs.
\Cref{sec:problem} gives the formal definitions.

\paragraph{(1) Identification at every corruption budget.}
For every $k\in\Nat\cup\{\Fin\}$ we characterize $\kCtrId{k}$ exactly (\cref{thm:main}).
A class is identifiable if and only if two conditions hold.
First, any two hypotheses with incomparable supports cover~$\X$ and share at least $2k+1$ points (infinitely many points if $k=\Fin$).
Second, every hypothesis has a finite tell-tale set that each proper sub-hypothesis misses in at least $2k+1$ points (for $k=\Fin$: no support contains another).
For $k=0$ the second condition is Angluin's, so the overlap condition is the only obstruction that the hidden direction adds to positive data (\cref{cor:clean}).
For two hypotheses $h,g$ the least budget at which they become confusable is
\[
   \rho(h,g)=\big\lceil\tfrac12\,\kb(h,g)\big\rceil ,
\]
where the \emph{joint defect} $\kb(h,g)$ counts the points that every presentation must pay for with a violation (\cref{thm:pair}, \cref{fig:threshold}).
The budgets form a strict hierarchy, and an unknown budget is strictly harder than every known one (\cref{cor:budget-hierarchy}).
A finite budget is also the only kind of count that a learner can tolerate. Under any unbounded budget, even one that grows like $\log\log n$, every two hypotheses share a presentation (\cref{prop:unbounded}). What helps beyond finite budgets is honest progress, not time (\cref{prop:progress}).

\paragraph{(2) Contrastive generation.}
A \emph{contrastive closure dimension}, the analogue for pair constraints of the closure dimension of \citet{lrt25}, characterizes uniform contrastive generation. When the dimension is positive, the optimal sample complexity is exactly the dimension plus one (\cref{thm:uniform-contrastive-generation}). A chain decomposition gives the non-uniform version (\cref{thm:nonuniform-contrastive-generation}).
For ordinary generation on finite classes, the obstruction is exact. A finite class is generatable if and only if every family of its members with a shared presentation has an infinite common support (\cref{thm:finite-generation}). On countable classes it is not: a class can fail although all its finite subclasses succeed (\cref{thm:countable-gap}). Compatible finite witnesses, transferred from the characterization of text generation by \citet{lhjg26b}, give a sufficient condition for every class (\cref{thm:witness}). On finite classes, every criterion in this paper depends only on the sizes of the cells of the membership map (\cref{rem:profile}).

\paragraph{(3) A diamond on clean data.}
On countable classes with infinite supports, the four criteria form a strict diamond, not a chain (\cref{thm:diamond}, \cref{fig:diamond}).
Contrastive identification is strictly contained in both contrastive generation and text identification, and both are strictly contained in text generation. Contrastive generation and text identification are incomparable.

\paragraph{(4) A reversal under corruption.}
On clean data the hidden direction costs: $\Ctr\Id\subsetneq\Txt\Id$.
Under corruption this inclusion fails.
The co-singleton class above lies in $\kCtrId{\Fin}$ but not in $\kTxtId{1}$ (\cref{thm:co-singleton}), and for every $k\ge1$ the classes $\kCtrId{k}$ and $\kTxtId{k}$ are incomparable (\cref{thm:corrupted-incomparability}).
By \cref{thm:main}, the classes with at least two members that survive every unknown finite budget are exactly those whose holes $\X\setminus\supp(h)$ are pairwise disjoint and leave infinitely many common positives.

\begin{figure}[t]
\centering
\begin{tikzpicture}[figfont, x=1cm, y=1cm,
  ctr/.style={rounded corners=4pt, draw=hc, fill=hc!9, line width=0.8pt, minimum width=46pt, minimum height=19pt, inner sep=3pt},
  txt/.style={rounded corners=4pt, draw=gc, fill=gc!9, line width=0.8pt, minimum width=46pt, minimum height=19pt, inner sep=3pt},
  wit/.style={font=\scriptsize\itshape, text=grayc, inner sep=3pt, sloped, pos=0.45}]
\node[txt] (top)   at (0, 2.3)   {$\Txt\Gen$};
\node[ctr] (left)  at (-3.6, 0)  {$\Ctr\Gen$};
\node[txt] (right) at (3.6, 0)   {$\Txt\Id$};
\node[ctr] (bot)   at (0, -2.3)  {$\Ctr\Id$};
\draw[arr, inkc!85, shorten <=2pt, shorten >=2pt] (left)  -- node[wit, above] {disjoint pair} (top);
\draw[arr, inkc!85, shorten <=2pt, shorten >=2pt] (right) -- node[wit, above] {punctured supports} (top);
\draw[arr, inkc!85, shorten <=2pt, shorten >=2pt] (bot)   -- node[wit, below, pos=0.58] {augmented supports} (left);
\draw[arr, inkc!85, shorten <=2pt, shorten >=2pt] (bot)   -- node[wit, below, pos=0.58] {disjoint pair} (right);
\draw[inkc!60, line width=0.7pt, dash pattern=on 4pt off 2.5pt] (left.east) -- node[font=\scriptsize, fill=white, inner sep=2pt] {incomparable} (right.west);
\end{tikzpicture}
\caption{The clean hierarchy on countable classes with infinite supports (\cref{thm:diamond}). Blue: contrastive data; red: positive data. Each arrow points to a strictly larger class and carries the name of a class that separates the two (\cref{tab:separations}). Contrastive generation and text identification are incomparable.}
\label{fig:diamond}
\end{figure}

\subsection{The idea: antipodal cells}
\label{subsec:overview}

All four results rest on one map.
For a finite family $\F$ of hypotheses, the \emph{membership map} $\Phi_\F(x):=(h(x))_{h\in\F}$ sends each example to a vertex of the cube $\{0,1\}^\F$. Its fibres split $\X$ into \emph{cells}.
A pair $\{x,y\}$ separates $x$ from $y$ under every $h\in\F$ exactly when $\Phi_\F(x)$ and $\Phi_\F(y)$ are \emph{antipodal}, that is, differ in every coordinate.
Thus pairs that are valid for every member of $\F$ can cover a point if and only if its antipodal cell is not empty (\cref{lem:antipodal}).
We read this one fact at three scales.
\begin{itemize}
  \item \emph{Families.} A finite family has a shared contrastive presentation if and only if every occupied cell outside the zero pattern has an occupied antipode (\cref{cor:shared}). Shared presentations with a finite common support block generation (\cref{prop:finite-family-obstruction}).
  \item \emph{Pairs.} Two hypotheses have four cells, and at most one of the cells $R_{10},R_{01},R_{00}$ is empty (\cref{lem:trichotomy}). The points that no such pair covers form the antipode of the empty cell. Their number is the joint defect $\kb(h,g)$ (\cref{prop:defect-formula}).
  \item \emph{Streams.} Every point without an antipodal partner costs a violation in every presentation, and the cost survives corruption up to a loss of $k$ (\cref{lem:charging}). This charging bound drives every robust learner in the paper.
\end{itemize}
The hidden direction also has an algebraic description. Over $\mathbb F_2$, the map $h\mapsto\Cut(h)$ that sends a hypothesis to its set of crossing pairs is linear, and its kernel is $\{0,\one\}$. So contrastive data fix the labels on each connected component of the observed graph up to one hidden bit (\cref{rem:hidden-bit}).

\paragraph{Organization.}
\Cref{sec:related} reviews related work and \cref{sec:problem} fixes the definitions.
\Cref{sec:crossings} develops the cell calculus.
\Cref{sec:identification} characterizes identification at every budget.
\Cref{sec:generation} treats generation and the diamond.
\Cref{sec:reversal} proves the reversal, and \cref{sec:conclusion} lists open problems.
The appendix contains full proofs of all results, a further discussion, and additional related work.

\section{Related Work}
\label{sec:related}

\paragraph{Identification in the limit.}
The paradigm starts with \citet{gold67}. The tell-tale theorem of \citet{angluin80} characterizes identification from positive data, and \citet{lzz08} survey the results for indexed families.
Robust and noise-tolerant variants go back to the approximate identification of \citet{wharton1974approximate} (\cref{app:extended-related}).
Recent work characterizes the limit learnability of recursive functions from evaluations on every domain point \citep{pf25}, a labeled and two-sided setting.
\citet{psv26} add computational traces of the accepting machine to Gold's model and obtain identifiability across the Chomsky hierarchy with different levels of corruption tolerance. This is another way around Gold's negative results.

\paragraph{Generation in the limit.}
\citet{km24} introduced generation in the limit and proved that it is possible for every countable class with infinite supports.
\citet{lrt25} restated the model in learning-theoretic notation and introduced the closure dimension, which characterizes uniform generation.
Later work studies breadth, density, the trade-off with noise, hallucination detection, generation in continuous spaces, agnostic and safe generation, differentially private generation, and closure under unions \citep{kmv24b,rr25,bpz25,kw25a,cp24,kmsv25,lrt26,hp26,akk26,lhjg26,mvyz26,hkmv25}.
Our contrastive closure dimension is the analogue of the closure dimension of \citet{lrt25} for edge sets: finite sets of pair constraints replace finite sets of positive examples.

\paragraph{Learning from pair signals.}
Our data model, pairs known to disagree but without labels, is a special case of pair-level supervision.
Cannot-link and must-link constraints guide clustering at least from \citet{wcrs01} on; a cannot-link constraint between two clusters is exactly an unoriented disagreement.

In weakly supervised classification, \citet{bns18} learn from pairwise similarity and unlabeled data, and \citet{dbs21} learn from similar and dissimilar pairs whose labels may be noisy.
Related settings include learning from positive and unlabeled data \citep{en08,bd20}, complementary-label learning \citep{inhs17}, learning from label proportions \citep{qsmcl09}, and multiple-instance learning with exactly one positive \citep{ccgg18}.
Closest in name is a recent learning-theoretic line on \emph{contrastive examples} in the active and self-directed query models \citep{mss+25,rrm+26}. There, a contrastive example is a pair of nearby points with different labels. The older comparison-query model \citep{klmz17} is a further relative.
These models also learn from two-point signals, but the learner knows which point carries which label, and they measure success by sample or query complexity.
In our setting the orientation is hidden, and we measure success in the limit.
More broadly, the works above study finite-sample learning with stochastic noise.
We study the asymptotic regime, where combinatorial structure, not concentration of measure, decides feasibility, and where an adversary, not a noise distribution, corrupts the data.

\section{Preliminaries}
\label{sec:problem}

\paragraph{Setting.}
We use the learning-theoretic notation of \citet{lrt25}, and we write $\Nat=\{0,1,2,\dots\}$.
The example space $\X$ is countably infinite, and we fix an enumeration $\X=\{u_0,u_1,\dots\}$; ``the least element'' of a subset refers to this enumeration.
A class $\Hyp\subseteq\{0,1\}^\X$ is a countable set of hypotheses.

The \emph{support} of $h$ is $\supp(h):=\{x\in\X:h(x)=1\}$, and we work extensionally: a hypothesis is its support.
We call $h$ \emph{proper nontrivial} if $\varnothing\ne\supp(h)\ne\X$.
Contrastive data exist only for such targets, because a target with empty or full support has no pair with different labels.
Two supports are \emph{incomparable} if neither contains the other.
We write $[\X]^2$ for the set of two-element subsets of $\X$ and call its elements \emph{pairs}.
For $h\in\{0,1\}^\X$, the \emph{crossing set} is
\[
   \Cut(h):=\big\{\{x,y\}\in[\X]^2:\ h(x)\ne h(y)\big\}.
\]
A pair is valid for $h$ exactly when it lies in $\Cut(h)$, that is, when exactly one of its two points is positive (the XOR condition $h(x)\oplus h(y)=1$).

\begin{definition}[Uniformly unbounded support \citep{lrt25}]
A class $\Hyp$ has \emph{uniformly unbounded support} (UUS) if $\supp(h)$ is infinite for every $h\in\Hyp$.
\end{definition}

\begin{definition}[Positive-data closure \citep{lrt25}]
For $x_{1:n}=(x_1,\dots,x_n)$, the version space is $\Hyp(x_{1:n}):=\{g\in\Hyp:\{x_1,\dots,x_n\}\subseteq\supp(g)\}$, and the closure is $\langle x_{1:n}\rangle_\Hyp:=\bigcap_{g\in\Hyp(x_{1:n})}\supp(g)$ if the version space is not empty, and $\bot$ otherwise.
\end{definition}

\begin{definition}[Presentations]
\label{def:presentations}
Let $h$ be a hypothesis and $k\in\Nat$.
\begin{itemize}
  \item A \emph{$k$-corrupted text} for $h$ is a sequence $Q=(x_t)_{t\ge1}$ in $\X$ such that $x_t\notin\supp(h)$ for at most $k$ indices $t$, and every point of $\supp(h)$ occurs in $Q$.
  \item An \emph{informant} for $h$ is a sequence $((x_t,h(x_t)))_{t\ge1}$ in which every point of $\X$ occurs.
  \item A \emph{$k$-corrupted contrastive presentation} for a proper nontrivial $h$ is a sequence $P=(p_t)_{t\ge1}$ of pairs such that $p_t\notin\Cut(h)$ for at most $k$ indices $t$, and every point of $\supp(h)$ lies in some $p_t$.
\end{itemize}
For $k=0$ we drop ``$0$-corrupted'' and speak of a \emph{(clean)} text or contrastive presentation.
In a contrastive presentation, a pair $p_t\in\Cut(h)$ is \emph{honest} and a pair $p_t\notin\Cut(h)$ is \emph{corrupt}.
Pairs are unordered and may repeat. We count corrupt pairs over indices, that is, with repetitions.
For any presentation $Q$ we write $Q_{\le n}$ for its prefix of length~$n$ and $\Seen_n(Q)$ for the set of points in $Q_{\le n}$; for a contrastive presentation, $\Seen_n(P)=\bigcup_{t\le n}p_t$.
For a hypothesis $g$, the \emph{violation count} of a finite sequence of pairs is $\viol_g(p_1,\dots,p_n):=|\{t\le n:\ p_t\notin\Cut(g)\}|$.
\end{definition}

Corruption changes only the validity of single observations: a corrupted presentation must still show every positive.
A contrastive presentation tells the learner that each pair crosses the cut of $h$, and it never tells the learner which endpoint is positive.

\begin{definition}[Learning in the limit \citep{gold67,km24}]
\label{def:learning-criteria}
Let $\mathsf X\in\{\Txt,\Inff,\Ctr\}$ and $k\in\Nat$.
\begin{itemize}
  \item An \emph{identifier} $I$ maps finite sequences of observations to $\Hyp$. It \emph{identifies} $h$ from a set of presentations if for every such presentation $Q$ of $h$ there is $N$ with $I(Q_{\le n})=h$ for all $n\ge N$.
  \item A \emph{generator} $G$ maps finite sequences of observations to $\X$. It \emph{generates} $h$ if for every presentation $Q$ of $h$ there is $N$ such that $G(Q_{\le n})\in\supp(h)\setminus\Seen_n(Q)$ for all $n\ge N$.
\end{itemize}
We write $\Hyp\in\mathsf X\Id$ (or $\mathsf X\Gen$) if one identifier (or generator) succeeds for every $h\in\Hyp$ on all clean presentations of type~$\mathsf X$.
We write $\Hyp\in\kCtrId{k}$ (or $\kTxtId{k}$) if one identifier succeeds for every $h\in\Hyp$ on all $k$-corrupted contrastive presentations (or texts); this identifier may depend on~$k$.
We write $\Hyp\in\kCtrId{\Fin}$ if one identifier succeeds for every $h\in\Hyp$ on all $k$-corrupted contrastive presentations, for every $k\in\Nat$; this identifier does not know~$k$.
\end{definition}

Every criterion passes to subclasses: if an identifier succeeds on $\Hyp$ and $\Hyp'\subseteq\Hyp$, replace each output outside $\Hyp'$ by a fixed member of $\Hyp'$.
Also $\kCtrId{\Fin}\subseteq\kCtrId{(k+1)}\subseteq\kCtrId{k}\subseteq\kCtrId{0}=\Ctr\Id$, because a presentation with fewer corrupt pairs is a special case.

We use three results from earlier work.

\begin{theorem}[Informant data \citep{gold67}]
\label{thm:gold-informant}
Every countable class lies in $\Inff\Id$.
\end{theorem}

\begin{theorem}[Tell-tale theorem \citep{angluin80}]
\label{thm:angluin}
$\Hyp\in\Txt\Id$ if and only if every $g\in\Hyp$ has a finite \emph{tell-tale} set $T_g\subseteq\supp(g)$ such that no $f\in\Hyp$ satisfies $T_g\subseteq\supp(f)\subsetneq\supp(g)$.
\end{theorem}

\begin{theorem}[Generation from text \citep{km24}]
\label{thm:KM}
Every countable class with UUS lies in $\Txt\Gen$.
\end{theorem}

\section{Common Crossings and Antipodal Cells}
\label{sec:crossings}

This section builds the one tool that the rest of the paper uses.
A pair of points is useful to an adversary when it is valid for several hypotheses at once, because then it does not help the learner to tell them apart.
We describe these pairs through the cells of the membership map, and we count the points that no such pair can cover.

\subsection{The membership map}

\begin{definition}[Common crossings and membership map]
\label{def:common-crossing}
Let $\F$ be a finite nonempty family of hypotheses.
\begin{itemize}
  \item The \emph{common crossing graph} is $\GammaC(\F):=\bigcap_{h\in\F}\Cut(h)$, the set of pairs that are valid for every member of $\F$. For two hypotheses we write $\GammaC(h,g)$.
  \item For a set $E$ of pairs, $\V(E):=\bigcup_{e\in E}e$ is the set of points that lie in some pair of $E$.
  \item The \emph{membership map} is $\Phi_\F:\X\to\{0,1\}^\F$, $\Phi_\F(x):=(h(x))_{h\in\F}$. Its fibres $R_\alpha:=\Phi_\F^{-1}(\alpha)$ are the \emph{cells}. The \emph{antipode} of a pattern $\alpha$ is $\anti\alpha:=\one-\alpha$.
\end{itemize}
\end{definition}

\begin{lemma}[Antipodal coverage]
\label{lem:antipodal}
Let $\F$ be a finite nonempty family of hypotheses.
\begin{enumerate}[label=\textnormal{(\alph*)}]
  \item A pair $\{x,y\}$ lies in $\GammaC(\F)$ if and only if $\Phi_\F(y)=\anti{\Phi_\F(x)}$. Thus $x\in\V(\GammaC(\F))$ if and only if the antipodal cell $R_{\anti{\Phi_\F(x)}}$ is not empty.
  \item For a nonempty set $S\subseteq\X$, some infinite sequence of pairs in $\GammaC(\F)$ covers $S$ if and only if $S\subseteq\V(\GammaC(\F))$.
\end{enumerate}
\end{lemma}

\begin{proof}
(a) The pair crosses $h$ if and only if $h(y)=1-h(x)$. This holds for every $h\in\F$ if and only if $\Phi_\F(y)=\one-\Phi_\F(x)$.
(b) If a sequence covers $S$, then each point of $S$ lies in one of its pairs. Conversely, choose for each $x\in S$ one pair $e_x\in\GammaC(\F)$ that contains $x$. List the pairs $e_x$ along an enumeration of $S$, and repeat the list forever if $S$ is finite.
\end{proof}

A sequence is a contrastive presentation for every $h\in\F$ exactly when all its pairs lie in $\GammaC(\F)$ and it covers $\bigcup_{h\in\F}\supp(h)$. We call such a sequence a \emph{shared presentation} of $\F$. \Cref{lem:antipodal}(b) with $S=\bigcup_{h\in\F}\supp(h)$ gives the following criterion; a point lies in some support exactly when its pattern is not the zero pattern~$\mathbf 0$.

\begin{corollary}[Shared presentations]
\label{cor:shared}
A finite nonempty family $\F$ of proper nontrivial hypotheses has a shared contrastive presentation if and only if, for every pattern $\alpha\ne\mathbf 0$, the cell $R_\alpha$ is empty or the cell $R_{\anti\alpha}$ is not empty.
\end{corollary}

\subsection{Two hypotheses}

For two hypotheses $h,g$ we always order the coordinates as $(h,g)$. The four cells are
\begin{align*}
   R_{11}&=\supp(h)\cap\supp(g), & R_{10}&=\supp(h)\setminus\supp(g),\\
   R_{01}&=\supp(g)\setminus\supp(h), & R_{00}&=\X\setminus\big(\supp(h)\cup\supp(g)\big).
\end{align*}
The antipodal pairs of cells are $\{R_{11},R_{00}\}$ and $\{R_{10},R_{01}\}$ (\cref{fig:cells}).
We say that $h$ and $g$ \emph{cover} $\X$ if $\supp(h)\cup\supp(g)=\X$, and that they form an \emph{overlapping cover} if, in addition, $\supp(h)\cap\supp(g)\ne\varnothing$.

\begin{figure}[t]
\centering
\begin{tikzpicture}[figfont, x=1cm, y=1cm]
\begin{scope}
  \draw[frameX] (-3.25,-1.85) rectangle (3.25,1.85);
  \node[font=\footnotesize, text=grayc, anchor=north east] at (3.2,1.8) {$\X$};
  \fill[hfill] (-0.85,0) ellipse [x radius=1.95cm, y radius=1.25cm];
  \fill[gfill] ( 0.85,0) ellipse [x radius=1.95cm, y radius=1.25cm];
  \begin{scope}
    \clip (-0.85,0) ellipse [x radius=1.95cm, y radius=1.25cm];
    \fill[bothfill] ( 0.85,0) ellipse [x radius=1.95cm, y radius=1.25cm];
  \end{scope}
  \draw[hborder] (-0.85,0) ellipse [x radius=1.95cm, y radius=1.25cm];
  \draw[gborder] ( 0.85,0) ellipse [x radius=1.95cm, y radius=1.25cm];
  \node[font=\footnotesize, text=hc] at (-1.55, 0.66) {$\supp(h)$};
  \node[font=\footnotesize, text=gc] at ( 1.55, 0.66) {$\supp(g)$};
  \node[pt] (a1) at ( 0.00, 0.80) {};
  \node[pt] (a2) at ( 0.00,-0.80) {};
  \node[pt] (b1) at (-2.15, 0.30) {};
  \node[pt] (c1) at ( 2.15, 0.30) {};
  \node[pt] (d1) at (-3.00, 1.55) {};
  \node[pt] (d2) at ( 3.00,-1.55) {};
  \draw[crossedge] (a1) -- (d1);
  \draw[crossedge] (a2) -- (d2);
  \draw[crossedge] (b1) -- (c1);
  \node[cell] at (0,-0.22) {$R_{11}$};
  \node[cell] at (-2.25,-0.40) {$R_{10}$};
  \node[cell] at ( 2.25,-0.40) {$R_{01}$};
  \node[cell] at (-2.80,-1.55) {$R_{00}$};
  \node[paneltitle] at (0,2.2) {(a) cells of $\Phi_{\{h,g\}}$};
\end{scope}
\begin{scope}[xshift=5.9cm, yshift=-1.25cm]
  \coordinate (s00) at (0,0);   \coordinate (s10) at (2.5,0);
  \coordinate (s01) at (0,2.5); \coordinate (s11) at (2.5,2.5);
  \draw[hc, line width=0.9pt, dash pattern=on 4pt off 2pt] (s00) -- (s10);
  \draw[hc, line width=0.9pt, dash pattern=on 4pt off 2pt] (s01) -- (s11);
  \draw[gc, line width=0.9pt, dash pattern=on 1.2pt off 1.8pt] (s00) -- (s01);
  \draw[gc, line width=0.9pt, dash pattern=on 1.2pt off 1.8pt] (s10) -- (s11);
  \draw[crossedge, line width=1.5pt] (s00) -- (s11);
  \draw[crossedge, line width=1.5pt] (s10) -- (s01);
  \foreach \v in {s00,s10,s01,s11} {
    \node[circle, fill=white, draw=inkc, line width=0.8pt, minimum size=9pt, inner sep=0pt] at (\v) {};
  }
  \node[cell, anchor=north east] at ([shift={(-3pt,-3pt)}]s00) {$R_{00}$};
  \node[cell, anchor=north west] at ([shift={(3pt,-3pt)}]s10) {$R_{10}$};
  \node[cell, anchor=south east] at ([shift={(-3pt,3pt)}]s01) {$R_{01}$};
  \node[cell, anchor=south west] at ([shift={(3pt,3pt)}]s11) {$R_{11}$};
  \node[font=\scriptsize, text=hc, below=3pt] at (1.25,0) {crosses $h$ only};
  \node[font=\scriptsize, text=gc, rotate=90, above=3pt] at (0,1.25) {crosses $g$ only};
  \node[font=\scriptsize, text=inkc, fill=white, inner sep=1.5pt] at (1.25,1.25) {crosses both};
  \node[paneltitle] at (1.25,3.45) {(b) antipodal pairs};
\end{scope}
\end{tikzpicture}
\caption{The four cells of two hypotheses $h,g$. (a)~A pair is valid for both $h$ and $g$ exactly when it joins $R_{11}$ to $R_{00}$ or $R_{10}$ to $R_{01}$ (solid edges). (b)~The same fact on the square $\{0,1\}^2$ of membership patterns: a pair crosses $h$ when it moves horizontally, crosses $g$ when it moves vertically, and crosses both exactly along the two diagonals, which join antipodal patterns (\cref{lem:antipodal,lem:trichotomy}).}
\label{fig:cells}
\end{figure}

\begin{lemma}[Pair trichotomy]
\label{lem:trichotomy}
Let $h\ne g$ be proper nontrivial.
\begin{enumerate}[label=\textnormal{(\alph*)}]
  \item $\GammaC(h,g)$ consists of the pairs that join $R_{11}$ to $R_{00}$ and the pairs that join $R_{10}$ to $R_{01}$. It is not empty.
  \item At most one of the cells $R_{10},R_{01},R_{00}$ is empty. Moreover, $R_{00}=\varnothing$ if and only if $h$ and $g$ cover $\X$; $R_{01}=\varnothing$ if and only if $\supp(g)\subsetneq\supp(h)$; and $R_{10}=\varnothing$ if and only if $\supp(h)\subsetneq\supp(g)$.
\end{enumerate}
\end{lemma}

\begin{proof}
(b) If $R_{10}=R_{01}=\varnothing$, then $h=g$. If $R_{10}=R_{00}=\varnothing$, then $\supp(g)=\X$. If $R_{01}=R_{00}=\varnothing$, then $\supp(h)=\X$. The three equivalences restate the definitions of the cells; the inclusions are strict because $h\ne g$.
(a) The first sentence is \cref{lem:antipodal}(a). The graph is empty only if one of $R_{11},R_{00}$ and one of $R_{10},R_{01}$ are empty. Empty $R_{11},R_{10}$ give $\supp(h)=\varnothing$, and empty $R_{11},R_{01}$ give $\supp(g)=\varnothing$. The other two combinations contradict (b).
\end{proof}

So exactly one of four cases holds for two distinct proper nontrivial hypotheses: they cover $\X$; $\supp(g)\subsetneq\supp(h)$; $\supp(h)\subsetneq\supp(g)$; or none of these.

\subsection{Defects}

\begin{definition}[Defect]
\label{def:defect}
Let $h\ne g$ be proper nontrivial. The \emph{defect} of $h$ against $g$ is the set of positives of $h$ that no common crossing covers, and the \emph{joint defect} is the corresponding set for both supports:
\[
   \Defect(h\to g):=\supp(h)\setminus\V(\GammaC(h,g)),\qquad
   \Defect(h,g):=\big(\supp(h)\cup\supp(g)\big)\setminus\V(\GammaC(h,g)).
\]
We write $\kappa(h\to g):=|\Defect(h\to g)|$ and $\kb(h,g):=|\Defect(h,g)|$, with values in $\{0,1,2,\dots\}\cup\{\infty\}$.
\end{definition}

A point of $\Defect(h\to g)$ is a positive of $h$ that every contrastive presentation of $h$ must cover with a pair that $g$ rejects. The next formula computes all defects. \Cref{fig:cases} shows the four cases.

\begin{figure}[t]
\centering
\begin{tikzpicture}[figfont, x=1cm, y=1cm,
  casenote/.style={font=\footnotesize, align=center, text=inkc},
  casedef/.style={font=\footnotesize, align=center, text=defc!85!black}]
\def\W{3.3}\def\H{2.2}
\begin{scope}
  \fill[hfill] (0,0) -- (2.05,0) .. controls (2.35,0.75) and (1.85,1.45) .. (2.1,\H) -- (0,\H) -- cycle;
  \fill[gfill] (1.25,0) .. controls (1.0,0.75) and (1.5,1.45) .. (1.2,\H) -- (\W,\H) -- (\W,0) -- cycle;
  \fill[deffill] (1.25,0) -- (2.05,0) .. controls (2.35,0.75) and (1.85,1.45) .. (2.1,\H) -- (1.2,\H) .. controls (1.5,1.45) and (1.0,0.75) .. (1.25,0) -- cycle;
  \draw[hborder] (2.05,0) .. controls (2.35,0.75) and (1.85,1.45) .. (2.1,\H);
  \draw[gborder] (1.25,0) .. controls (1.0,0.75) and (1.5,1.45) .. (1.2,\H);
  \draw[frameX] (0,0) rectangle (\W,\H);
  \node[cell] at (0.6,1.1) {$R_{10}$};
  \node[cell] at (1.66,1.1) {$R_{11}$};
  \node[cell] at (2.75,1.1) {$R_{01}$};
  \node[casenote, anchor=north] at (\W/2,-0.12) {$h,g$ cover $\X$\\ $R_{00}=\varnothing$};
  \node[casedef, anchor=north] at (\W/2,-0.88) {defect $R_{11}$, both directions};
\end{scope}
\begin{scope}[xshift=3.85cm]
  \fill[deffill] (1.55,1.1) ellipse [x radius=1.3cm, y radius=0.88cm];
  \fill[white] (1.95,1.05) ellipse [x radius=0.6cm, y radius=0.46cm];
  \fill[bothfill] (1.95,1.05) ellipse [x radius=0.6cm, y radius=0.46cm];
  \draw[hborder] (1.55,1.1) ellipse [x radius=1.3cm, y radius=0.88cm];
  \draw[gborder] (1.95,1.05) ellipse [x radius=0.6cm, y radius=0.46cm];
  \draw[frameX] (0,0) rectangle (\W,\H);
  \node[cell] at (0.8,1.15) {$R_{10}$};
  \node[cell] at (1.95,1.05) {$R_{11}$};
  \node[cell] at (3.0,0.22) {$R_{00}$};
  \node[casenote, anchor=north] at (\W/2,-0.12) {$\supp(g)\subsetneq\supp(h)$\\ $R_{01}=\varnothing$};
  \node[casedef, anchor=north] at (\W/2,-0.88) {defect $R_{10}$, from $h$ to $g$};
\end{scope}
\begin{scope}[xshift=7.7cm]
  \fill[deffill] (1.75,1.1) ellipse [x radius=1.3cm, y radius=0.88cm];
  \fill[white] (1.35,1.05) ellipse [x radius=0.6cm, y radius=0.46cm];
  \fill[bothfill] (1.35,1.05) ellipse [x radius=0.6cm, y radius=0.46cm];
  \draw[gborder] (1.75,1.1) ellipse [x radius=1.3cm, y radius=0.88cm];
  \draw[hborder] (1.35,1.05) ellipse [x radius=0.6cm, y radius=0.46cm];
  \draw[frameX] (0,0) rectangle (\W,\H);
  \node[cell] at (2.5,1.15) {$R_{01}$};
  \node[cell] at (1.35,1.05) {$R_{11}$};
  \node[cell] at (0.3,0.22) {$R_{00}$};
  \node[casenote, anchor=north] at (\W/2,-0.12) {$\supp(h)\subsetneq\supp(g)$\\ $R_{10}=\varnothing$};
  \node[casedef, anchor=north] at (\W/2,-0.88) {defect $R_{01}$, from $g$ to $h$};
\end{scope}
\begin{scope}[xshift=11.55cm]
  \fill[hfill] (1.15,1.1) ellipse [x radius=0.92cm, y radius=0.72cm];
  \fill[gfill] (2.15,1.1) ellipse [x radius=0.92cm, y radius=0.72cm];
  \begin{scope}
    \clip (1.15,1.1) ellipse [x radius=0.92cm, y radius=0.72cm];
    \fill[bothfill] (2.15,1.1) ellipse [x radius=0.92cm, y radius=0.72cm];
  \end{scope}
  \draw[hborder] (1.15,1.1) ellipse [x radius=0.92cm, y radius=0.72cm];
  \draw[gborder] (2.15,1.1) ellipse [x radius=0.92cm, y radius=0.72cm];
  \draw[frameX] (0,0) rectangle (\W,\H);
  \node[cell] at (0.72,1.1) {$R_{10}$};
  \node[cell] at (1.65,1.1) {$R_{11}$};
  \node[cell] at (2.6,1.1) {$R_{01}$};
  \node[cell] at (3.0,0.22) {$R_{00}$};
  \node[casenote, anchor=north] at (\W/2,-0.12) {otherwise\\ no empty cell};
  \node[casedef, anchor=north] at (\W/2,-0.88) {no defect};
\end{scope}
\end{tikzpicture}
\caption{The four cases of the defect formula (\cref{prop:defect-formula}). Blue outline: $\supp(h)$; red outline: $\supp(g)$. In each case at most one of $R_{10},R_{01},R_{00}$ is empty, and the defect (amber) is the antipode of the empty cell: its points have no partner that crosses both cuts. The joint defect $\kb(h,g)$ is the size of the amber region.}
\label{fig:cases}
\end{figure}

\begin{proposition}[Defect formula]
\label{prop:defect-formula}
For distinct proper nontrivial $h,g$, exactly one of the following four cases holds.
\begin{center}
\begin{tabular}{@{}lll@{}}
\toprule
case & $\Defect(h\to g)$ & $\Defect(g\to h)$\\
\midrule
$h$ and $g$ cover $\X$ & $\supp(h)\cap\supp(g)$ & $\supp(h)\cap\supp(g)$\\
$\supp(g)\subsetneq\supp(h)$ & $\supp(h)\setminus\supp(g)$ & $\varnothing$\\
$\supp(h)\subsetneq\supp(g)$ & $\varnothing$ & $\supp(g)\setminus\supp(h)$\\
none of the above & $\varnothing$ & $\varnothing$\\
\bottomrule
\end{tabular}
\end{center}
In particular, the joint defect lies inside one cell, and $\kb(h,g)=\max\{\kappa(h\to g),\kappa(g\to h)\}$.
\end{proposition}

\begin{proof}
By \cref{lem:antipodal}(a), a point lies outside $\V(\GammaC(h,g))$ exactly when its antipodal cell is empty.
The positives of $h$ lie in $R_{11}$, with antipode $R_{00}$, and in $R_{10}$, with antipode $R_{01}$. The positives of $g$ lie in $R_{11}$ and in $R_{01}$, with antipode $R_{10}$.
By \cref{lem:trichotomy}(b) at most one of $R_{10},R_{01},R_{00}$ is empty. So the four cases of the table are: $R_{00}$ is empty, $R_{01}$ is empty, $R_{10}$ is empty, or none is empty.
In each case the defect consists of the positives whose antipodal cell is the empty one, as the table shows.
\end{proof}

The defect formula also describes when clean data can rule out one hypothesis. Following the earlier literature we say that $g$ is \emph{eliminable from $h$} if every clean contrastive presentation for $h$ contains a pair outside $\Cut(g)$.

\begin{corollary}[Eliminability]
\label{cor:eliminability}
Let $h\ne g$ be proper nontrivial.
\begin{enumerate}[label=\textnormal{(\alph*)}]
  \item $g$ is not eliminable from $h$ if and only if $\kappa(h\to g)=0$. This happens in exactly three cases: $\supp(h)\subsetneq\supp(g)$; the supports are disjoint; or the supports are incomparable, intersect, and do not cover $\X$.
  \item $h$ and $g$ have a shared clean presentation if and only if $\kb(h,g)=0$, that is, if and only if their supports are incomparable and they do not form an overlapping cover.
\end{enumerate}
\end{corollary}

\begin{proof}
(a) A clean presentation for $h$ avoids pairs outside $\Cut(g)$ exactly when all its pairs lie in $\GammaC(h,g)$. By \cref{lem:antipodal}(b) such a presentation exists if and only if $\supp(h)\subseteq\V(\GammaC(h,g))$, that is, $\kappa(h\to g)=0$. The three cases are the cases of \cref{prop:defect-formula} with an empty first column; the first row has an empty defect exactly when the supports are disjoint.
(b) The same argument with $S=\supp(h)\cup\supp(g)$; the joint defect is empty exactly in the last row and in the first row with disjoint supports.
\end{proof}

\subsection{Charging}

Defect points make violations unavoidable. The next lemma counts them, also when the stream contains corrupt pairs.

\begin{lemma}[Charging]
\label{lem:charging}
Let $h\ne g$ be proper nontrivial, $k\in\Nat$, and $P$ a $k$-corrupted contrastive presentation for $h$. If $W\subseteq\Defect(h\to g)$ is finite and $W\subseteq\Seen_n(P)$, then
\[
   \viol_g(P_{\le n})\ \ge\ |W|-k .
\]
Consequently, if $\kappa(h\to g)\ge 2k+1$, then $\viol_g(P_{\le n})\ge k+1$ for all large $n$; and if $\kappa(h\to g)=\infty$, then $\viol_g(P_{\le n})\to\infty$.
\end{lemma}

\begin{proof}
For $x\in W$ let $p_{t(x)}$ be the first pair of $P$ that contains $x$; then $t(x)\le n$.
If $p_{t(x)}$ is honest, it lies in $\Cut(h)$ but not in $\GammaC(h,g)$, because $x$ has no common crossing. So it lies outside $\Cut(g)$. An honest pair contains exactly one positive of $h$, so different points of $W$ have different honest first pairs.
A corrupt pair is the first pair of at most two points of $W$. If it is the first pair of two points, then both points lie in one cell by \cref{prop:defect-formula}. So $g$ gives them the same label, and the pair lies outside $\Cut(g)$.

Let $y_1$ and $y_2$ count the corrupt pairs in $P_{\le n}$ that are the first pair of exactly one and of exactly two points of $W$. Then $y_1+y_2\le k$ and
\[
   \viol_g(P_{\le n})\ \ge\ \big(|W|-y_1-2y_2\big)+y_2\ \ge\ |W|-k .
\]
For the consequences, note that $P$ covers $\supp(h)\supseteq W$, so every finite $W$ lies in $\Seen_n(P)$ for all large $n$. Also $\viol_g(P_{\le n})$ does not decrease in $n$.
\end{proof}

For clean presentations the bound is sharp. Cover each point of $\Defect(h\to g)$ once by a pair of $\Cut(h)$, cover every other positive of $h$ by a common crossing, and repeat a common crossing if the list is finite; common crossings exist by \cref{lem:trichotomy}(a), and $g$ rejects none of them.

When the holes $\X\setminus\supp(h)$ and $\X\setminus\supp(g)$ are disjoint, the charging can be done pair by pair. Every negative of $h$ is then a positive of $g$, so an honest pair for $h$ is rejected by $g$ exactly when its positive endpoint is also a positive of $g$.

\begin{lemma}[Charging with disjoint holes]
\label{lem:disjoint-charging}
Let $h\ne g$ be hypotheses whose holes $\X\setminus\supp(h)$ and $\X\setminus\supp(g)$ are disjoint, and let $p_1,\dots,p_n$ be pairs. Let $P_n$ be the set of positives of $h$ that lie in some pair $p_t\in\Cut(h)$ with $t\le n$. Then
\[
   \viol_g(p_1,\dots,p_n)\ \ge\ |P_n\cap\supp(g)|.
\]
\end{lemma}

\begin{proof}
A pair in $\Cut(h)$ has the form $\{p,q\}$ with $p\in\supp(h)$ and $q$ in the hole of $h$. The holes are disjoint, so $q\in\supp(g)$, and the pair lies outside $\Cut(g)$ exactly when $p\in\supp(g)$. Distinct points of $P_n\cap\supp(g)$ therefore come from distinct pairs outside $\Cut(g)$.
\end{proof}

\begin{remark}[One hidden bit per component]
\label{rem:hidden-bit}
Identify a hypothesis with a vector in $\mathbb F_2^{\X}$ and write $h\oplus g$ for the pointwise sum. A pair $\{x,y\}$ lies in $\Cut(h)$ exactly when $h(x)+h(y)=1$ in $\mathbb F_2$, and this expression is linear in $h$. So $\Cut(h\oplus g)$ is the symmetric difference of $\Cut(h)$ and $\Cut(g)$, and $\Cut(g)=\Cut(h)$ exactly when $g\in\{h,\one-h\}$. In the language of cochains, $\Cut$ is the coboundary map of the complete graph on $\X$, and its kernel $\{0,\one\}$ is precisely the hidden direction; the coverage condition on positives is what separates $h$ from $\one-h$.

Two consequences follow. First, $\GammaC(h,g)=\Cut(h)\setminus\Cut(h\oplus g)$: the common crossings of $h$ and $g$ are the crossings of $h$ that stay on one side of the set where $h$ and $g$ disagree, which is \cref{lem:trichotomy}(a). Second, let $E\subseteq\Cut(h)$ be a finite set of observed pairs. Then $E\subseteq\Cut(g)$ if and only if $h\oplus g$ is constant on each connected component of the graph $(\V(E),E)$. So the observations fix the labels on each component up to one hidden bit. A text or an informant has no such bits. In the disjoint-support example of \cref{sec:introduction}, the first $n$ pairs form a matching with $n$ components.
\end{remark}

\section{Identification at Every Corruption Budget}
\label{sec:identification}

We now characterize contrastive identification for every corruption budget $k\in\Nat\cup\{\Fin\}$.
The clean case is $k=0$.
We first treat two hypotheses, where the defect formula gives the answer directly, and then whole classes, where Angluin's tell-tale sets enter.

\subsection{Two hypotheses}

\begin{theorem}[Two hypotheses]
\label{thm:pair}
Let $h\ne g$ be proper nontrivial.
\begin{enumerate}[label=\textnormal{(\alph*)}]
  \item For $k\in\Nat$, the following are equivalent:
  \textnormal{(i)} $\{h,g\}\in\kCtrId{k}$;
  \textnormal{(ii)} no sequence is a $k$-corrupted contrastive presentation of both $h$ and $g$;
  \textnormal{(iii)} $\kb(h,g)\ge 2k+1$.
  Hence the least budget at which $h$ and $g$ share a corrupted presentation is
  \[
     \rho(h,g)=\big\lceil\tfrac12\,\kb(h,g)\big\rceil\qquad(\rho=\infty\text{ if }\kb=\infty).
  \]
  \item $\{h,g\}\in\kCtrId{\Fin}$ if and only if $\kappa(h\to g)=\kappa(g\to h)=\infty$, that is, if and only if $h$ and $g$ cover $\X$ and $\supp(h)\cap\supp(g)$ is infinite.
\end{enumerate}
\end{theorem}

\begin{proof}[Proof sketch]
(a)~A shared presentation forces the same outputs under both targets, so (i) implies (ii).
If $\kb(h,g)\le 2k$, then the joint defect lies in one cell (\cref{prop:defect-formula}), and each of its points has one covering pair that only $h$ rejects and one that only $g$ rejects. Charge half of the defect to each hypothesis and cover all other points by common crossings. This gives a shared $k$-corrupted presentation, so (ii) implies (iii).
If $\kb(h,g)\ge2k+1$, then by \cref{prop:defect-formula} one direction, say $h\to g$, carries the whole joint defect. The learner outputs $h$ once $\viol_g$ exceeds $k$, and $g$ otherwise. Under $g$ the count stays at most $k$; under $h$ it eventually exceeds $k$ by \cref{lem:charging}.
(b)~If both defects are infinite, the learner that outputs the hypothesis with fewer violations succeeds by \cref{lem:charging}.
Otherwise either $\kb<\infty$, and (a) gives a shared presentation at budget $\rho$, or the supports are nested with an infinite difference. In the nested case an adversary builds a clean presentation of the larger hypothesis that, at infinitely many times, also extends to a corrupted presentation of the smaller one, so a learner that does not know the budget names the smaller one infinitely often.
\Cref{app:proofs-id} gives the full proof.
\end{proof}

The factor $\tfrac12$ has a simple source. Each point of the joint defect costs one violation, and the adversary can charge this violation to either hypothesis. Two budgets of $k$ therefore absorb $2k$ defect points. \Cref{fig:threshold} summarizes \cref{thm:pair}.

\begin{figure}[t]
\centering
\begin{tikzpicture}[figfont, x=0.62cm, y=0.62cm,
  yes/.style={circle, fill=acc, inner sep=0pt, minimum size=4.6pt},
  no/.style={circle, draw=black!40, fill=white, line width=0.6pt, inner sep=0pt, minimum size=4.6pt}]
\fill[acc!12] (0.5,-0.5) -- (9.5,-0.5) -- (9.5,4.5) -- (8.5,4.5) -- (8.5,3.5) -- (6.5,3.5) -- (6.5,2.5)
  -- (4.5,2.5) -- (4.5,1.5) -- (2.5,1.5) -- (2.5,0.5) -- (0.5,0.5) -- cycle;
\draw[acc, line width=0.9pt] (0.5,-0.5) -- (0.5,0.5) -- (2.5,0.5) -- (2.5,1.5) -- (4.5,1.5) -- (4.5,2.5)
  -- (6.5,2.5) -- (6.5,3.5) -- (8.5,3.5) -- (8.5,4.5) -- (9.5,4.5);
\fill[acc!12] (10.5,-0.5) rectangle (11.5,4.5);
\foreach \c in {0,...,9} {
  \foreach \r in {0,...,4} {
    \pgfmathtruncatemacro{\ok}{\c >= 2*\r+1 ? 1 : 0}
    \ifnum\ok=1 \node[yes] at (\c,\r) {}; \else \node[no] at (\c,\r) {}; \fi
  }
  \node[no] at (\c,6) {};
  \node[font=\scriptsize, text=black!60] at (\c,-1.05) {\c};
}
\foreach \r in {0,...,4} {
  \node[yes] at (11,\r) {};
  \node[font=\scriptsize, text=black!60] at (-1.0,\r) {\r};
}
\fill[white] (11,6) circle (0.19);
\fill[acc] (11,6) -- ++(90:0.19) arc (90:-90:0.19) -- cycle;
\draw[black!40, line width=0.6pt] (11,6) circle (0.19);
\node[font=\scriptsize, text=black!60] at (11,-1.05) {$\infty$};
\node[font=\scriptsize, text=black!60] at (-1.1,6) {$\Fin$};
\draw[black!25, line width=0.4pt] (-0.5,5.25) -- (11.5,5.25);
\node[font=\footnotesize] at (5.5,-1.9) {joint defect $\kb(h,g)$};
\node[font=\footnotesize, rotate=90] at (-2.0,2.5) {budget $k$};
\node[font=\scriptsize, text=acc!85!black, anchor=west, align=left] at (12.4,2.0) {boundary $\kb=2k+1$:\\ $\rho(h,g)=\lceil\kb/2\rceil$};
\draw[acc!70, -{Stealth[length=1.5mm]}] (12.35,2.0) -- (8.6,3.9);
\node[font=\scriptsize, anchor=west, align=left] at (12.4,6.0) {half dot: identifiable only if\\ both defects are infinite};
\draw[black!50, -{Stealth[length=1.5mm]}] (12.35,6.0) -- (11.3,6.0);
\node[yes] at (12.6,0.4) {}; \node[font=\scriptsize, anchor=west] at (12.9,0.4) {$\{h,g\}$ identifiable};
\node[no]  at (12.6,-0.4) {}; \node[font=\scriptsize, anchor=west] at (12.9,-0.4) {shared presentation exists};
\end{tikzpicture}
\caption{When can a learner separate two hypotheses (\cref{thm:pair})? A filled point means $\{h,g\}\in\kCtrId{k}$. For a known budget $k$ this holds exactly when $\kb(h,g)\ge2k+1$ (shaded region). For an unknown budget (row $\Fin$), the joint defect must be infinite in \emph{both} directions, which by \cref{prop:defect-formula} means that the supports cover $\X$ and share infinitely many points.}
\label{fig:threshold}
\end{figure}

Part~(b) shows that knowing the budget matters already for two hypotheses. Let $\supp(h)$ be the even numbers and $\supp(g)=\Nat\setminus\{1\}$. Then $\kb(h,g)=\infty$, so $\{h,g\}\in\kCtrId{k}$ for every $k$, with a learner that uses $k$. But no single learner works for all budgets: the honest pairs $\{c,1\}$ with $c$ odd, which $g$ needs, look to $h$ exactly like corrupt pairs, and a learner that does not know $k$ cannot tell a corrupted presentation of the subset from a clean presentation of the superset.

\subsection{Whole classes}

For a class, pairwise conditions are not enough even without corruption: Angluin's theorem shows that infinite chains of sub-hypotheses can defeat every learner. The right condition is a tell-tale set that tolerates the loss of $2k$ points.
For $k\in\Nat$ let $\theta_k:=2k+1$, and let $\theta_\Fin:=\infty$, where ``at least $\infty$ points'' means ``infinitely many points''.

\begin{theorem}[Identification at every budget]
\label{thm:main}
Let $\Hyp$ be a countable class of proper nontrivial hypotheses and $k\in\Nat\cup\{\Fin\}$. Then $\Hyp\in\kCtrId{k}$ if and only if both of the following hold.
\begin{enumerate}[label=\textnormal{(\Roman*)}]
  \item \emph{(Incomparable pairs.)} Any two members of $\Hyp$ with incomparable supports cover $\X$ and share at least $\theta_k$ points.
  \item \emph{(Robust tell-tales.)} Every $g\in\Hyp$ has a finite set $T_g\subseteq\supp(g)$ such that every $f\in\Hyp$ with $\supp(f)\subsetneq\supp(g)$ misses at least $\theta_k$ points of $T_g$.
\end{enumerate}
\end{theorem}

For $k=\Fin$, condition (II) cannot hold with a finite $T_g$ if some $f$ lies below $g$; so it says that no support in $\Hyp$ contains another. Condition (II) also gives $\kb(f,g)\ge|T_g\setminus\supp(f)|\ge\theta_k$ for nested pairs, so together with (I) every pair satisfies the condition of \cref{thm:pair}. The theorem therefore splits identifiability into the pairwise threshold and a robust form of Angluin's condition.

\begin{proof}[Proof sketch]
\emph{Necessity.}
Condition (I) and, for $k=\Fin$, condition (II) follow from \cref{thm:pair} applied to the subclass $\{h,g\}$.
For finite $k$, suppose that some $g$ has no robust tell-tale. A compactness argument gives one fixed set $L\subseteq\supp(g)$ with at most $2k$ points. It also gives members $f_1,f_2,\dots$ below $g$ that miss $L$ and contain larger and larger parts of $\supp(g)\setminus L$.
The adversary covers $L$ once. This costs at most $k$ violations for $g$ and at most $k$ for every $f_j$. Then it alternates two moves. First, it extends the stream as a corrupted presentation of some $f_j$ until the learner names $f_j$. Second, it covers one more point of $g$. The limit is a $k$-corrupted presentation of $g$ on which the learner changes its mind infinitely often.

\emph{Sufficiency for $k\in\Nat$.}
The learner outputs the least-index $g$ that is \emph{eligible}: the learner has seen $T_g$, and $\viol_g\le k$.
The target is eventually eligible forever.

An earlier rival below the target, or incomparable to it, has a defect of at least $2k+1$ points by (I) and (II), so \cref{lem:charging} eventually makes it ineligible.
An earlier rival $g$ above the target is never eligible. Its tell-tale contains at least $2k+1$ points outside the target's support. An honest pair shows at most one of them, and $g$ rejects it. A corrupt pair shows at most two, and $g$ rejects it if it shows two. So the stream can show at most $\viol_g+k\le 2k$ of these points.

\emph{Sufficiency for $k=\Fin$.}
The learner outputs the least-index hypothesis with the fewest violations.
By (I) and (II) the holes $\X\setminus\supp(g)$ are pairwise disjoint.
An honest pair with positive endpoint $p$ is then rejected by a rival $g$ exactly when $p\in\supp(g)$.
A rival keeps at most $k_0$ violations, where $k_0$ is the true number of corrupt pairs, only if most seen positives lie in its hole. The holes are disjoint, so at most one rival can do this, and it eventually fails because it shares infinitely many positives with the target.
\Cref{app:proofs-id} gives the full proof.
\end{proof}

The second condition of \cref{thm:main} is needed, even when every pair passes the threshold of \cref{thm:pair}:

\begin{example}[Pairwise thresholds do not suffice]
\label{ex:pairwise-not-enough}
Let $Y\subseteq\X$ be infinite with $\X\setminus Y$ infinite, fix two points $\ell_1,\ell_2\in Y$, and list $Y\setminus\{\ell_1,\ell_2\}=\{y_1,y_2,\dots\}$. Let $\supp(g)=Y$ and $\supp(f_j)=\{y_1,\dots,y_{3j}\}$ for $j\ge1$.
All supports are nested, and any two differ in at least $3$ points, so every pair has $\kb\ge3$ and is identifiable at budget $1$ by \cref{thm:pair}. The class lies in $\Ctr\Id$, with tell-tale $T_g=\{\ell_1\}$.
But it is not in $\kCtrId{1}$: every finite $T\subseteq Y$ lies, up to the two points $\ell_1,\ell_2$, inside some $f_j$, so (II) fails for $g$. Concretely, the pairs $\{\ell_1,d\}$ and $\{\ell_2,y_1\}$ with $d\notin Y$ cost one violation for $g$ and one for every $f_j$, and after them the adversary can follow a presentation of $g$ that looks like a presentation of $f_j$ for larger and larger $j$.
\end{example}

\begin{corollary}[Clean identification]
\label{cor:clean}
For a countable class $\Hyp$ of proper nontrivial hypotheses, the following are equivalent:
\textnormal{(i)} $\Hyp\in\Ctr\Id$;
\textnormal{(ii)} $\Hyp\in\Txt\Id$ and no two members of $\Hyp$ have a shared contrastive presentation;
\textnormal{(iii)} $\Hyp\in\Txt\Id$ and every two members with incomparable supports form an overlapping cover.
\end{corollary}

\begin{proof}
For $k=0$ we have $\theta_0=1$. Condition (I) says that incomparable members form overlapping covers, and condition (II) is Angluin's condition, which is equivalent to $\Hyp\in\Txt\Id$ by \cref{thm:angluin}. So \cref{thm:main} gives (i)$\iff$(iii). By \cref{cor:eliminability}(b), two members share a clean presentation exactly when their supports are incomparable and do not form an overlapping cover, so (ii)$\iff$(iii).
\end{proof}

So the overlap condition is the only obstruction that the hidden direction adds to positive data.
The simplest class that it excludes is the disjoint pair of \cref{sec:introduction}:

\begin{corollary}
\label{cor:ctrid-txtid}
On countable classes of proper nontrivial hypotheses, $\Ctr\Id\subseteq\Txt\Id$, and the inclusion is strict, even on UUS classes.
\end{corollary}

\begin{proof}
The inclusion is \cref{cor:clean}. For strictness, let $A$ and $B$ be disjoint infinite sets. The class $\{h_A,h_B\}$ with supports $A$ and $B$ lies in $\Txt\Id$ (the first example names the target) but not in $\Ctr\Id$, because its supports are disjoint.
\end{proof}

The restriction to proper nontrivial hypotheses is needed. The hypothesis $\one$ with full support has no contrastive presentation, so the class $\Hco\cup\{\one\}$, where $\Hco$ is the co-singleton class of \cref{sec:reversal}, lies in $\Ctr\Id$ by \cref{thm:co-singleton}. It is not in $\Txt\Id$, because every finite set lies in some $\supp(h_s)\subsetneq\X$, so $\one$ has no tell-tale set.

\begin{remark}[Guessing the orientation does not help]
\label{rem:no-orientation-reduction}
One might reduce contrastive to text identification by guessing which endpoint of each pair is positive and running a text learner. \Cref{cor:ctrid-txtid} rules this out: the stream $\{a_1,b_1\},\{a_2,b_2\},\dots$ is a valid presentation for both $h_A$ and $h_B$, and no assignment of orientations can serve both targets.
\end{remark}

The budgets themselves form a strict hierarchy.

\begin{corollary}[Budget hierarchy]
\label{cor:budget-hierarchy}
\[
   \Ctr\Id=\kCtrId{0}\ \supsetneq\ \kCtrId{1}\ \supsetneq\ \kCtrId{2}\ \supsetneq\ \cdots\ \supsetneq\ \bigcap_{k\in\Nat}\kCtrId{k}\ \supsetneq\ \kCtrId{\Fin}.
\]
\end{corollary}

\begin{proof}
The inclusions hold because fewer corrupt pairs are a special case (\cref{sec:problem}).
For $\kCtrId{k}\ne\kCtrId{(k+1)}$, take $h,g$ that cover $\X$ and share exactly $2k+1$ points. Then $\kb(h,g)=2k+1$ (\cref{prop:defect-formula}), so $\{h,g\}$ lies in $\kCtrId{k}$ but not in $\kCtrId{(k+1)}$ by \cref{thm:pair}(a).
The last inclusion is strict by the example after \cref{thm:pair}.
\end{proof}

\begin{remark}[Effectivity]
\label{rem:effectivity}
For finite $k$, the learner in the proof of \cref{thm:main} uses the robust tell-tale sets as an oracle. Like Angluin's theorem, the statement is information-theoretic.
For $k=\Fin$, the learner needs no oracle beyond membership. The proof shows that eventually every rival has more violations than the target. So the following rule also works: among the first $n$ hypotheses, output the least index with the fewest violations. For an indexed family with decidable membership this rule is computable.
\end{remark}

\section{Generation: Closure and Confusability}
\label{sec:generation}

A generator does not need to name the target. It must eventually output positives of the target that it has not seen.
The theory has two layers. A closure dimension governs \emph{uniform} generation exactly. Ordinary generation has two sufficient conditions, safe cores and eventual cores, and one obstruction, confusable families.

\subsection{Edge-induced closure and the contrastive closure dimension}
\label{subsec:edge-closure}

The proof of \cref{thm:KM} uses positives that the text confirms. Contrastive data confirm only that two points have different labels, so we replace the version space of positive examples by a version space of edge constraints.

\begin{definition}[Edge-induced closure]
\label{def:edge-induced}
For a finite set $E$ of pairs, let $\HDelta(E):=\{g\in\Hyp:E\subseteq\Cut(g)\}$.
The \emph{contrastive closure} is $\langle E\rangle^{\Cut}_{\Hyp}:=\bigcap_{g\in\HDelta(E)}\supp(g)$ if $\HDelta(E)\ne\varnothing$, and $\bot$ otherwise.
For a contrastive prefix $P_{\le n}$, let $E_n(P):=\{p_t:t\le n\}$ be the set of distinct observed pairs, and let $\Safe_n(P):=\langle E_n(P)\rangle^{\Cut}_\Hyp$.
\end{definition}

If $P_{\le n}$ is valid for $h$, then $h\in\HDelta(E_n(P))$ and so $\Safe_n(P)\subseteq\supp(h)$. A point of $\Safe_n(P)\setminus\Seen_n(P)$ is therefore a \emph{certified new positive}.

\begin{definition}[Uniform and non-uniform contrastive generation]
A generator $G$ is a \emph{uniform contrastive generator with sample complexity $d$} if for every $h\in\Hyp$, every sequence $P$ of pairs in $\Cut(h)$, and every $n$ with $|E_n(P)|\ge d$, we have $G(P_{\le n})\in\supp(h)\setminus\Seen_n(P)$.
The class is \emph{non-uniformly} contrastively generatable if one generator $G$ has, for each $h\in\Hyp$, some $d_h<\infty$ that works in this sense for~$h$.
\end{definition}

The sample complexity counts \emph{distinct} pairs: a stream that repeats one pair forever never triggers the guarantee. Novelty refers to the seen \emph{points}. The two counts differ, and we keep them apart.

\begin{definition}[Hollow edge sets; contrastive closure dimension]
\label{def:hollow}
A finite set $E$ of pairs is \emph{hollow} for $\Hyp$ if $\HDelta(E)\ne\varnothing$ and $\langle E\rangle^{\Cut}_\Hyp\subseteq\V(E)$.
The \emph{contrastive closure dimension} $\CDelta(\Hyp)\in\Nat\cup\{\infty\}$ is the supremum of $|E|$ over hollow sets $E$, and $0$ if no set is hollow.
\end{definition}

A hollow set is a finite prefix after which the learner has already seen every forced positive. The dimension measures how long an adversary can keep the learner without a certified new positive.

\begin{theorem}[Uniform contrastive generation]
\label{thm:uniform-contrastive-generation}
$\Hyp$ is uniformly contrastively generatable if and only if $\CDelta(\Hyp)<\infty$. If $1\le\CDelta(\Hyp)=d<\infty$, then the optimal sample complexity is $d+1$; if $\CDelta(\Hyp)=0$, sample complexity $1$ suffices.
\end{theorem}

\begin{proof}
\emph{Sufficiency.} Let $d:=\CDelta(\Hyp)<\infty$ and $E:=E_n(P)$. If $|E|\le d$ or $\HDelta(E)=\varnothing$, output any point. Otherwise output the least point of $\langle E\rangle^{\Cut}_\Hyp\setminus\V(E)$.
For a target $h$ and a sequence $P$ in $\Cut(h)$ we have $h\in\HDelta(E)$, so the closure lies in $\supp(h)$. If $|E|>d$, then $E$ is not hollow, so the closure has a point outside $\V(E)=\Seen_n(P)$, and the output is a new positive of $h$.

\emph{Necessity.} Suppose that a generator $G$ has sample complexity $d'$, and that some nonempty hollow set $E^*$ has $|E^*|\ge d'$. Present the pairs of $E^*$ in any order, and let $x$ be the output of $G$ after them. If $x\in\V(E^*)$, the output is not new. Otherwise $x\notin\langle E^*\rangle^{\Cut}_\Hyp$, so some $h\in\HDelta(E^*)$ has $x\notin\supp(h)$; the sequence $E^*$ continued by any pairs of $\Cut(h)$ is valid for $h$, and $G$ fails on it. So every nonempty hollow set has fewer than $d'$ pairs. Hence $\CDelta(\Hyp)<\infty$, and if $\CDelta(\Hyp)=d\ge1$, then some hollow set has $d$ pairs, and $d'\ge d+1$.
\end{proof}

The case $\CDelta(\Hyp)=0$ is different: one pair suffices, but this bound need not be tight. For example, on $\X=\Nat$ the class whose only member is the set of even numbers has no hollow set, and the generator that outputs the least unseen even number is correct on every prefix, including the empty one.

\begin{remark}[The dimension is not the positive-data dimension]
\label{rem:not-collapse}
$\CDelta$ is not the closure dimension of \citet{lrt25} with pairs in place of points. Take disjoint infinite sets $A=\{a_i\}$ and $B=\{b_i\}$ and the class $\{h_A,h_B\}$ with these supports. One positive example names the target, so the positive-data closure dimension is $0$. But each pair $\{a_i,b_i\}$ crosses both cuts, so $E_n=\{\{a_i,b_i\}:i\le n\}$ has $\HDelta(E_n)=\{h_A,h_B\}$ and closure $A\cap B=\varnothing$. So $E_n$ is hollow for every $n$, and $\CDelta=\infty$.
 In the language of \cref{rem:hidden-bit}, $E_n$ is a matching with $n$ components, and its hidden bits are never resolved. The hidden direction, not the pair format, raises the dimension.
\end{remark}

\begin{remark}[Lifting to pairs is a different task]
\label{rem:lifting}
One can lift each $h$ to the set $\Cut(h)$ of its crossing pairs and ask for generation in the sense of \citet{km24} over the example space $[\X]^2$. That task differs from ours in two ways. Its output is a new crossing pair, which does not name a positive point. Its input is a text that enumerates all of $\Cut(h)$, but a contrastive presentation needs to cover only $\supp(h)$.
The disjoint pair separates the two tasks. The lifted sets $\Cut(h_A)$ and $\Cut(h_B)$ are infinite, so the lifted class is generatable from texts by \cref{thm:KM}. But $\{h_A,h_B\}\notin\Ctr\Gen$, by \cref{prop:finite-family-obstruction} below.
\end{remark}

The non-uniform version follows by a standard threshold-and-defer argument.

\begin{theorem}[Non-uniform contrastive generation]
\label{thm:nonuniform-contrastive-generation}
$\Hyp$ is non-uniformly contrastively generatable if and only if $\Hyp=\bigcup_m\Hyp_m$ for a chain $\Hyp_1\subseteq\Hyp_2\subseteq\cdots$ with $\CDelta(\Hyp_m)<\infty$ for every $m$.
\end{theorem}

The generator in the proof (\cref{app:proofs-gen}) takes the chain and the dimensions $\CDelta(\Hyp_m)$ as input, so the result is information-theoretic.

The dimension does not capture ordinary contrastive generation. A class can rely on an \emph{eventual core} instead of a uniform safe set, and then $\CDelta$ can be infinite.

\begin{example}[Punctured supports]
\label{ex:punctured}
Let $A=\{a_m\}_{m\ge1}$ be an infinite set with $A\ne\X$, listed without repetition, and let $b\in\X\setminus A$. Let $\supp(h_\infty)=A$ and $\supp(h_m)=A\setminus\{a_m\}$ for $m\ge1$.
The sequence $(a_m)$ is an eventual core (\cref{def:eventual-core}), so the class lies in $\Ctr\Gen$ by \cref{prop:eventual-core}.
But $E_n=\{\{a_i,b\}:i\le n\}$ is hollow: $\HDelta(E_n)$ consists of $h_\infty$ and the $h_m$ with $m>n$, whose supports intersect in $\{a_1,\dots,a_n\}\subseteq\V(E_n)$. So $\CDelta=\infty$.
\end{example}

\subsection{Cores and confusable families}
\label{subsec:cores-confusability}

Ordinary generation asks only that each target is eventually handled. Two sufficient conditions are natural. A \emph{safe core} makes the contrastive closure infinite after finitely many pairs, and an \emph{eventual core} is one sequence whose tail eventually enters every support.

\begin{proposition}[Safe cores]
\label{prop:safe-sufficiency}
Suppose that for every $h\in\Hyp$ and every contrastive presentation $P$ of $h$, the set $\Safe_n(P)$ is infinite for some $n$. Then $\Hyp\in\Ctr\Gen$.
\end{proposition}

\begin{proof}
Output the least point of $\Safe_n(P)\setminus(\Seen_n(P)\cup\{\hat x_1,\dots,\hat x_{n-1}\})$, where $\hat x_t$ are the earlier outputs, if this set is not empty, and an arbitrary point otherwise. Every output of the first kind lies in $\supp(h)$ because $\Safe_n(P)\subseteq\supp(h)$. As $n$ grows, $E_n(P)$ grows and $\HDelta(E_n(P))$ shrinks, so $\Safe_n(P)$ grows. By assumption it is therefore infinite for all large $n$, and from then on the excluded set is finite, so the output is a new positive of $h$.
\end{proof}

\begin{example}[Augmented supports]
\label{ex:augmented}
Let $A$ be infinite with $\X\setminus A=\{b_m\}_{m\ge1}$ infinite, listed without repetition. Let $\supp(h_\infty)=A$ and $\supp(h_m)=A\cup\{b_m\}$.
Every member contains $A$, so $\Safe_n(P)\supseteq A$, and the class lies in $\Ctr\Gen$ by \cref{prop:safe-sufficiency}.
It is not in $\Ctr\Id$. For $m\ne m'$, the supports of $h_m$ and $h_{m'}$ are incomparable and miss every $b_{m''}$ with $m''\notin\{m,m'\}$. So they do not cover $\X$, and \cref{cor:clean} applies.
\end{example}

\begin{definition}[Eventual core]
\label{def:eventual-core}
An injective sequence $(r_m)_{m\ge1}$ in $\X$ is an \emph{eventual core} for $\Hyp$ if, for every $h\in\Hyp$, only finitely many $m$ satisfy $r_m\notin\supp(h)$.
\end{definition}

\begin{proposition}[Eventual cores]
\label{prop:eventual-core}
A countable class of proper nontrivial hypotheses with infinite supports and an eventual core lies in $\Ctr\Gen$.
\end{proposition}

\begin{proof}
At step $n$, output $r_m$ for the least $m\ge n$ such that $r_m$ is neither seen nor output before. The tail $\{r_m:m\ge n\}$ is infinite and the excluded set is finite, so the output exists and is new. For a target $h$, all but finitely many $r_m$ lie in $\supp(h)$, and the chosen index satisfies $m\ge n\to\infty$. So the output eventually lies in $\supp(h)$.
\end{proof}

The obstruction comes from shared presentations of whole families. We call a finite nonempty family $\F\subseteq\Hyp$ \emph{confusable} if it has a shared contrastive presentation. By \cref{cor:shared} this is a condition on the cells of $\Phi_\F$: every occupied cell outside the zero pattern has an occupied antipode. Subfamilies of confusable families are confusable, so the confusable families form a simplicial complex.

\begin{proposition}[Finite-family obstruction]
\label{prop:finite-family-obstruction}
If $\Hyp$ has a confusable family $\F$ with $|\F|\ge2$ and $\big|\bigcap_{h\in\F}\supp(h)\big|<\infty$, then $\Hyp\notin\Ctr\Gen$.
\end{proposition}

\begin{proof}
Let $P$ be a shared presentation of $\F$. A deterministic generator produces the same outputs on $P$ for every target in $\F$. If it succeeds for every $h\in\F$, then after the largest of the convergence times all outputs lie in $\bigcap_{h\in\F}\supp(h)$. This set is finite, and $P$ covers it, so eventually the generator has seen all of its points. From then on no new output can lie in it, a contradiction.
\end{proof}

Under UUS a single hypothesis is never an obstruction, because its support is infinite.
Pairwise analysis does not suffice:

\begin{example}[An obstruction of order three]
\label{ex:higher-order-obstruction}
Split $\X$ into six infinite cells with the patterns $100,010,001,110,101,011$ for $(h_1,h_2,h_3)$, so that $R_{000}=R_{111}=\varnothing$.
Each pairwise intersection of supports is infinite; for example, $\supp(h_1)\cap\supp(h_2)=R_{110}$. The triple intersection is $R_{111}=\varnothing$.
The occupied cells come in antipodal pairs ($100\leftrightarrow011$, $010\leftrightarrow101$, $001\leftrightarrow110$), so $\{h_1,h_2,h_3\}$ is confusable by \cref{cor:shared}, and \cref{prop:finite-family-obstruction} rules out contrastive generation.
\end{example}

On finite classes the obstruction of \cref{prop:finite-family-obstruction} is exact. The proof rests on one observation. For a finite nonempty family $K$, call
\[
   D(K):=\big\{g\in K:\ \supp(g)\subseteq\V(\GammaC(K))\big\}
\]
the \emph{coverable core} of $K$: the members whose positives all have a partner that crosses every cut of $K$.

\begin{lemma}[Coverable core]
\label{lem:coverable-core}
For every finite nonempty family $K\subseteq\Hyp$ of proper nontrivial hypotheses, the coverable core $D(K)$ is empty or confusable.
\end{lemma}

\begin{proof}
Assume that $D(K)$ is not empty. It is a subfamily of $K$, so every pair that crosses all cuts of $K$ also crosses all cuts of $D(K)$, that is, $\GammaC(K)\subseteq\GammaC(D(K))$. By the definition of $D(K)$, every positive of a member of $D(K)$ lies in $\V(\GammaC(K))$, hence in $\V(\GammaC(D(K)))$, and \cref{lem:antipodal}(b) gives a sequence of pairs of $\GammaC(D(K))$ that covers all these positives. This sequence is a shared presentation of $D(K)$.
\end{proof}

\begin{theorem}[Ordinary generation on finite classes]
\label{thm:finite-generation}
A finite class $\Hyp$ of proper nontrivial hypotheses lies in $\Ctr\Gen$ if and only if every confusable family $\F\subseteq\Hyp$ has an infinite common support $\bigcap_{h\in\F}\supp(h)$.
\end{theorem}

\begin{proof}[Proof sketch]
Necessity is \cref{prop:finite-family-obstruction}; for a single hypothesis, a finite support leaves no new positive.
For sufficiency, the generator tracks the set $K$ of hypotheses with no violation so far and outputs new points of the common support of its coverable core $D(K)$. By \cref{lem:coverable-core} a nonempty core is confusable, so its common support is infinite.
The class is finite, so $K$ stops changing after some time. From then on the target lies in $D(K)$: the presentation itself consists of pairs of $\GammaC(K)$, and it covers the support of the target. \Cref{app:proofs-gen} gives the full proof.
\end{proof}

The generator does not follow every hypothesis that stays consistent; it follows only those that the common crossings can cover. The next example shows that the difference matters.

\begin{example}[Consistency is not confusability]
\label{ex:consistency}
Split $\X$ into six infinite cells with the patterns $110,101,010,001,011,000$ for $(h,g_1,g_2)$, so that $R_{111}$ and $R_{100}$ are empty (\cref{fig:cube}(b)).
The cells of $\supp(h)$, namely $R_{110}$ and $R_{101}$, have the occupied antipodes $R_{001}$ and $R_{010}$. So some clean presentation of $h$ uses only pairs that cross all three cuts, and on it all three hypotheses stay consistent forever, although their common support $R_{111}$ is empty.
Still, the class lies in $\Ctr\Gen$. The triple is not confusable, because the occupied cell $R_{011}$ has the empty antipode $R_{100}$. The other confusable families are the single hypotheses and the three pairs, whose common supports $R_{110}$, $R_{101}$, and $R_{011}$ are infinite. So \cref{thm:finite-generation} applies.
\end{example}

\begin{figure}[t]
\centering
\begin{tikzpicture}[figfont, x=1cm, y=1cm,
  occ/.style={circle, fill=inkc, inner sep=0pt, minimum size=6pt},
  emp/.style={circle, draw=grayc, fill=white, line width=0.7pt, dash pattern=on 1.5pt off 1pt, inner sep=0pt, minimum size=7pt},
  vlab/.style={font=\scriptsize, inner sep=3pt},
  cubeedge/.style={lightc, line width=0.6pt},
  diag/.style={inkc, line width=1.2pt},
  brokendiag/.style={defc, line width=1.1pt, dash pattern=on 3pt off 2pt}]
\newcommand{\cubeat}[2]{%
  \foreach \a in {0,1} { \foreach \b in {0,1} { \foreach \c in {0,1} {
    \coordinate (#1\a\b\c) at ({#2+2.4*\a+1.1*\c},{2.4*\b+0.68*\c});
  }}}
  \foreach \b in {0,1} { \foreach \c in {0,1} { \draw[cubeedge] (#10\b\c) -- (#11\b\c); }}
  \foreach \a in {0,1} { \foreach \c in {0,1} { \draw[cubeedge] (#1\a0\c) -- (#1\a1\c); }}
  \foreach \a in {0,1} { \foreach \b in {0,1} { \draw[cubeedge] (#1\a\b0) -- (#1\a\b1); }}
}
\cubeat{L}{0}
\draw[diag] (L100) -- (L011);
\draw[diag] (L010) -- (L101);
\draw[diag] (L001) -- (L110);
\foreach \v in {100,010,001,110,101,011} { \node[occ] at (L\v) {}; }
\foreach \v in {000,111} { \node[emp] at (L\v) {}; }
\node[vlab, below left]  at (L000) {$000$};
\node[vlab, below right] at (L100) {$100$};
\node[vlab, above left]  at (L010) {$010$};
\node[vlab, above left]  at (L110) {$110$};
\node[vlab, below right] at (L001) {$001$};
\node[vlab, below right] at (L101) {$101$};
\node[vlab, above left]  at (L011) {$011$};
\node[vlab, above right, text=defc] at (L111) {$111=\varnothing$};
\node[paneltitle] at (1.75,3.8) {(a) \cref{ex:higher-order-obstruction}};
\node[panelnote] at (1.75,-0.8) {confusable, empty top cell:\\ not generatable};
\cubeat{R}{7.0}
\draw[diag] (R110) -- (R001);
\draw[diag] (R101) -- (R010);
\draw[brokendiag] (R011) -- (R100);
\foreach \v in {110,101,010,001,011,000} { \node[occ] at (R\v) {}; }
\foreach \v in {100,111} { \node[emp] at (R\v) {}; }
\node[vlab, below left]  at (R000) {$000$};
\node[vlab, below right, text=defc] at (R100) {$100=\varnothing$};
\node[vlab, above left]  at (R010) {$010$};
\node[vlab, above left]  at (R110) {$110$};
\node[vlab, below right] at (R001) {$001$};
\node[vlab, below right] at (R101) {$101$};
\node[vlab, above left]  at (R011) {$011$};
\node[vlab, above right, text=grayc] at (R111) {$111=\varnothing$};
\node[paneltitle] at (8.75,3.8) {(b) \cref{ex:consistency}};
\node[panelnote] at (8.75,-0.8) {not confusable ($011$ has no partner):\\ generatable};
\end{tikzpicture}
\caption{Antipodal cells of three hypotheses. Each vertex of the cube $\{0,1\}^3$ is a membership pattern; filled vertices are occupied cells and dashed circles are empty cells. A family has a shared presentation exactly when every occupied vertex other than $000$ has an occupied antipode (\cref{cor:shared}); solid diagonals join occupied antipodal pairs. (a)~All three diagonals are present, so the family is confusable, but its common support $R_{111}$ is empty, and generation fails. (b)~The diagonal from $011$ ends at an empty cell, so the three hypotheses never share a presentation, although a clean presentation of $h$ keeps all three consistent forever; by \cref{thm:finite-generation} generation succeeds.}
\label{fig:cube}
\end{figure}

\begin{remark}[The cell profile decides]
\label{rem:profile}
For a finite class $\Hyp$, call the map $\alpha\mapsto|R_\alpha|$ on $\{0,1\}^\Hyp$, with values in $\Nat\cup\{\infty\}$, the \emph{cell profile} of $\Hyp$.
Every set in our criteria is a union of cells of $\Phi_\Hyp$: supports, their intersections and differences, and the cells of every subfamily.
So by \cref{cor:shared} and \cref{thm:finite-generation}, membership of $\Hyp$ in $\Ctr\Gen$ depends only on which cells are empty and which are infinite.
On a finite class, the robust tell-tale sets of \cref{thm:main} exist exactly when each proper sub-support misses at least $2k+1$ points. So membership in $\kCtrId{k}$ depends only on the cell sizes truncated at $2k+1$ for finite $k$, and only on which cells are empty or infinite for $k=\Fin$.
On finite classes the theory is therefore a combinatorics of weighted cubes, and the profile decides every criterion. On countable classes this fails: the profiles of all finite subclasses do not decide generation (\cref{thm:countable-gap,rem:profile-gap}).
\end{remark}

\subsection{Countable classes and finite witnesses}
\label{subsec:countable-gen}

On countable classes the condition of \cref{thm:finite-generation} stays necessary, but it is no longer sufficient. The reason is \cref{ex:consistency} with a top cell that moves.

\begin{theorem}[Finite subclasses do not decide generation]
\label{thm:countable-gap}
There is a countable class $\Hyp$ of proper nontrivial hypotheses with infinite supports with the following properties. Every confusable family $\F\subseteq\Hyp$ has an infinite common support, so every finite subclass of $\Hyp$ lies in $\Ctr\Gen$. But $\Hyp\notin\Ctr\Gen$.
\end{theorem}

\noindent\emph{The class.} Split $\X$ into an infinite set $C=\{c_1,c_2,\dots\}$, four further infinite sets $A_1,A_2,B_1,B_2$, and a nonempty set $N$. Let
\[
   \supp(a)=C\cup A_1\cup A_2,\qquad \supp(b)=C\cup B_1\cup B_2,\qquad \supp(s_m)=\{c_1,\dots,c_m\}\cup A_1\cup B_1
\]
for $m\ge1$, and $\Hyp:=\{a,b,s_1,s_2,\dots\}$ (\cref{fig:gap}). For each $m$, the triple $(s_m,a,b)$ has the cell pattern of \cref{ex:consistency}, except that its top cell $\{c_1,\dots,c_m\}$ is finite but not empty.

\begin{proof}[Proof sketch]
A family without $a$ has $B_1$ in its common support, and a family without $b$ has $A_1$. The family $\{a,b\}$ has the common support $C$. A family $\{a,b\}\cup S$ with $S$ a nonempty set of hypotheses $s_m$ is not confusable. Indeed, a point $c_k$ with $k$ above every index in $S$ lies in $\supp(a)\cap\supp(b)$ and outside every member of $S$. Its antipodal cell is empty, because every $s_m$ lies inside $\supp(a)\cup\supp(b)$. So every confusable family has an infinite common support, and \cref{thm:finite-generation} applies to every finite subclass.
Now let $G$ be any generator. An adversary works in phases. Phase $m$ shows the pair $\{c_m,z\}$ for a fixed $z\in N$. Then it shows pairs $\{\alpha,\beta'\}$ and $\{\beta,\alpha'\}$ with $\alpha\in A_1$, $\beta'\in B_2$, $\beta\in B_1$, and $\alpha'\in A_2$, along fixed enumerations, until, after at least one new pair of each kind, $G$ outputs a new point of $A_1\cup B_1$. Every pair so far crosses $s_m$, so a phase that never ended would be a clean presentation of $s_m$. Once $c_1,\dots,c_m$ have appeared, the only new positives of $s_m$ lie in $A_1\cup B_1$. So each phase ends. The limit is a clean presentation of both $a$ and $b$, and $G$ outputs infinitely many points of $A_1\cup B_1$. Infinitely many of them miss $\supp(b)$, or infinitely many miss $\supp(a)$. \Cref{app:proofs-gen} gives the full proof.
\end{proof}

\begin{figure}[t]
\centering
\begin{tikzpicture}[figfont, x=1cm, y=1cm,
  blk/.style={minimum width=1.32cm, minimum height=0.56cm, inner sep=0pt, draw=lightc, line width=0.5pt},
  hdr/.style={font=\footnotesize, anchor=south, inner sep=2pt},
  rowlab/.style={font=\small, anchor=east, inner sep=4pt},
  arcl/.style={font=\scriptsize, fill=white, inner sep=1pt}]
\def\cols{0/{$c_1,\dots,c_m$}, 1.5/{$c_{m+1},\dots$}, 3/{$A_1$}, 4.5/{$A_2$}, 6/{$B_1$}, 7.5/{$B_2$}, 9/{$N$}}
\foreach \x/\lab in \cols { \node[hdr] at (\x,0.32) {\lab}; }
\foreach \x/\on in {0/1, 1.5/1, 3/1, 4.5/1, 6/0, 7.5/0, 9/0} {
  \ifnum\on=1 \node[blk, fill=hc!45] at (\x,0) {}; \else \node[blk, fill=white] at (\x,0) {}; \fi }
\foreach \x/\on in {0/1, 1.5/1, 3/0, 4.5/0, 6/1, 7.5/1, 9/0} {
  \ifnum\on=1 \node[blk, fill=gc!45] at (\x,-0.62) {}; \else \node[blk, fill=white] at (\x,-0.62) {}; \fi }
\foreach \x/\on in {0/1, 1.5/0, 3/1, 4.5/0, 6/1, 7.5/0, 9/0} {
  \ifnum\on=1 \node[blk, fill=tc!45] at (\x,-1.24) {}; \else \node[blk, fill=white] at (\x,-1.24) {}; \fi }
\node[rowlab] at (-0.7,0)     {$a$};
\node[rowlab] at (-0.7,-0.62) {$b$};
\node[rowlab] at (-0.7,-1.24) {$s_m$};
\draw[defc, line width=1pt, rounded corners=2pt] (0.82,-1.58) rectangle (2.18,0.3);
\coordinate (bot) at (0,-1.6);
\draw[pairedge] (0,-1.55) to[out=-70, in=-110] node[arcl, pos=0.5] {$\{c_k,z\}$, $k\le m$} (9,-1.55);
\draw[corruptedge] (1.5,-1.55) to[out=-60, in=-120] node[arcl, pos=0.5] {$\{c_{m+1},z\}$: next phase} (9,-1.55);
\draw[pairedge] (3,-1.55) to[out=-55, in=-125] node[arcl, pos=0.5] {$\{\alpha,\beta'\}$} (7.5,-1.55);
\draw[pairedge] (6,-1.55) to[out=-110, in=-70] node[arcl, pos=0.5] {$\{\beta,\alpha'\}$} (4.5,-1.55);
\node[font=\scriptsize, text=defc!85!black, anchor=south] at (1.5,0.62) {no partner in $\GammaC(a,b,s_m)$};
\end{tikzpicture}
\caption{The class of \cref{thm:countable-gap}. A filled cell means that the block lies in the support. Solid arcs are the pairs of phase $m$; they cross all three cuts, so the stream looks like a presentation of $s_m$, and the generator must output a new point of $A_1\cup B_1$. The dashed arc starts the next phase and eliminates $s_m$. The framed block $c_{m+1},c_{m+2},\dots$ has no partner that crosses all three cuts, so $\{a,b,s_m\}$ is never confusable, although its common support $\{c_1,\dots,c_m\}$ is finite. In the limit the stream is a clean presentation of both $a$ and $b$.}
\label{fig:gap}
\end{figure}

So contrastive generation is not a finitary property. Text generation behaves differently. Every countable class of infinite languages is generatable from text \citep{km24}. So for classes of infinite languages, a failure from text needs an uncountable class, such as the class of languages of density one, all of whose countable subfamilies are generatable \citep{lhjg26b}. Contrastive generation can fail already for a countable class all of whose finite subclasses succeed.

\begin{remark}[Profiles do not decide]
\label{rem:profile-gap}
Add to $\X$ a further infinite set $D$, and let $\Hyp'$ consist of $s_1,s_2,\dots$ and the hypotheses $a',b'$ with supports $\supp(a)\cup D$ and $\supp(b)\cup D$. On every finite subfamily, the points of $D$ have the pattern of the points $c_k$ with large $k$, so every finite subclass of $\Hyp'$ has the same cell profile as the matching subclass of $\Hyp$. Yet $\Hyp'\in\Ctr\Gen$. Let $K$ be the set of consistent hypotheses. Output a new point of $B_1$ if $a'\notin K$; else of $A_1$ if $b'\notin K$; else of $C$ if $K$ contains no $s_m$; else of $A_1\cup B_1$. Every branch that a target $s_m$ can reach outputs positives of $s_m$. A presentation of $a'$ must cover $D$, and the first pair that does so eliminates $b'$ (a partner in $B_1\cup B_2$) or every $s_m$ (a partner in $N$); from then on the outputs lie in $\supp(a')$. The target $b'$ is symmetric. So the profiles of all finite subclasses do not decide contrastive generation.
\end{remark}

What the condition misses is the whole family that is active at one time, not its finite confusable parts. \citet{lhjg26b} characterize generation from text by finite positive witnesses whose active families have infinite common intersections. The sufficient half of their criterion transfers to contrastive data.

\begin{definition}[Contrastive witnesses]
\label{def:witness}
A \emph{witness assignment} maps each $h\in\Hyp$ to a finite set $T(h)\subseteq\supp(h)$. For a finite set $E$ of pairs, the \emph{active family} is
\[
   \mathcal A_T(E):=\big\{h\in\Hyp:\ E\subseteq\Cut(h)\ \text{and}\ T(h)\subseteq\V(E)\big\},
\]
the hypotheses that are consistent with $E$ and whose witnesses have appeared. The assignment is \emph{compatible} if, for every finite $E$ with $\mathcal A_T(E)\ne\varnothing$, the common support of $\mathcal A_T(E)$ is not contained in $\V(E)$.
\end{definition}

\begin{theorem}[Finite witnesses]
\label{thm:witness}
\leavevmode
\begin{enumerate}[label=\textnormal{(\alph*)}]
  \item Every class with a compatible witness assignment lies in $\Ctr\Gen$, with a generator that depends only on the set of distinct observed pairs.
  \item A finite class of proper nontrivial hypotheses has a compatible witness assignment if and only if every confusable family has an infinite common support.
\end{enumerate}
\end{theorem}

\begin{proof}
(a) Let $g(E)$ be the least point of $\bigcap_{h\in\mathcal A_T(E)}\supp(h)\setminus\V(E)$ if the active family is not empty, and the least point outside $\V(E)$ otherwise. On a clean presentation of $h$, every observed pair lies in $\Cut(h)$, and the finite set $T(h)$ eventually appears. From then on $h$ is active, so $g$ outputs points of $\supp(h)$ that have not appeared.
(b) The condition is necessary by (a) and \cref{thm:finite-generation}. Conversely, for every $h$ and every $K\subseteq\Hyp$ with $h\in K$ and $\supp(h)\not\subseteq\V(\GammaC(K))$, choose a point $x_{h,K}\in\supp(h)\setminus\V(\GammaC(K))$, and let $T(h)$ consist of these finitely many points. Fix $E$ with $\mathcal A_T(E)\ne\varnothing$, and let $K$ be the set of hypotheses consistent with $E$. If some active $h$ had $\supp(h)\not\subseteq\V(\GammaC(K))$, then $x_{h,K}\in\V(E)$, and the pair of $E$ that contains it lies in $\GammaC(K)$, a contradiction. So the active family lies inside the coverable core $D(K)$, which is confusable by \cref{lem:coverable-core}. Its common support therefore contains the infinite common support of $D(K)$.
\end{proof}

The class of \cref{thm:countable-gap} has no compatible assignment, and the reason is visible in \cref{fig:gap}. For any finite witnesses of $a$ and $b$, choose $m$ larger than the indices of their points in $C$. All witness points of $a$, $b$, and $s_m$ then have partners that cross all three cuts. Let $E$ consist of one such pair for each witness point and the pairs $\{c_k,z\}$ for $k\le m$. Then $\{a,b,s_m\}$ is active, and the common support of the active family lies inside $\{c_1,\dots,c_m\}\subseteq\V(E)$. Every finite confusable family is harmless, but the active family contains the triple $\{a,b,s_m\}$, which is not confusable and has a finite core. The same contrast between finite subfamilies and the full active family drives the analysis of \citet{lhjg26b}.

\subsection{The clean hierarchy}
\label{subsec:clean-hierarchy}

The results so far settle how the four clean criteria relate. The answer is a diamond, not a chain.

\begin{theorem}[Diamond]
\label{thm:diamond}
On countable UUS classes of proper nontrivial hypotheses,
\[
   \Ctr\Id\subsetneq\Ctr\Gen\subsetneq\Txt\Gen,\qquad \Ctr\Id\subsetneq\Txt\Id\subsetneq\Txt\Gen,
\]
and $\Ctr\Gen$ and $\Txt\Id$ are incomparable.
\end{theorem}

\begin{proof}[Proof sketch]
The inclusions $\Ctr\Gen\subseteq\Txt\Gen$ and $\Txt\Id\subseteq\Txt\Gen$ hold because $\Txt\Gen$ contains every countable UUS class (\cref{thm:KM}). The inclusion $\Ctr\Id\subseteq\Txt\Id$ is \cref{cor:ctrid-txtid}. For $\Ctr\Id\subseteq\Ctr\Gen$, output new points of the support of the current guess. \Cref{tab:separations} lists the separating classes; \cref{app:proofs-gen} gives the details.
\end{proof}

\begin{table}[b]
\centering
\small
\begin{tabularx}{\linewidth}{@{}l>{\raggedright\arraybackslash}p{0.27\linewidth}X@{}}
\toprule
separation & witness & reason\\
\midrule
$\Ctr\Gen\not\subseteq\Ctr\Id$ & augmented supports (\cref{ex:augmented}) & safe core $A$; incomparable supports that do not cover $\X$\\
$\Txt\Id\not\subseteq\Ctr\Gen$ & disjoint pair $\{h_A,h_B\}$ & one example names the target; a shared presentation with empty common support (\cref{prop:finite-family-obstruction})\\
$\Ctr\Gen\not\subseteq\Txt\Id$ & punctured supports (\cref{ex:punctured}) & eventual core; $h_\infty$ has no tell-tale set\\
\bottomrule
\end{tabularx}
\caption{The three classes that separate the diamond. The disjoint pair also separates $\Ctr\Id$ from $\Txt\Id$ and $\Ctr\Gen$ from $\Txt\Gen$; the punctured class also separates $\Txt\Id$ from $\Txt\Gen$.}
\label{tab:separations}
\end{table}

\section{The Robustness Reversal}
\label{sec:reversal}

On clean data the hidden direction costs: $\Ctr\Id\subsetneq\Txt\Id$ (\cref{cor:ctrid-txtid}).
Under corruption the comparison changes.
A corrupted text example looks exactly like a real one.
A corrupted pair can break a structure that all honest pairs share, and the defect counts of \cref{sec:crossings} measure how much of this structure survives.
\Cref{thm:main} already says which classes survive which budgets. This section compares them with text identification.

\subsection{The co-singleton class}

\begin{definition}[Co-singleton class]
For $s\in\X$ let $h_s$ have support $\X\setminus\{s\}$. The \emph{co-singleton class} is $\Hco:=\{h_s:s\in\X\}$.
\end{definition}

\begin{theorem}[Text fragility]
\label{thm:text-fragile}
$\Hco\notin\kTxtId{1}$.
\end{theorem}

\begin{proof}
The enumeration of $\X$ is a $1$-corrupted text for every $h_s$: its only wrong example is $s$. A learner sees the same input for every target.
\end{proof}

In a contrastive presentation of $h_s$, every honest pair contains $s$, because $s$ is the only negative. So a natural rule counts, for each point, the pairs that miss it.

\begin{figure}[t]
\centering
\begin{tikzpicture}[figfont, x=1cm, y=1cm,
  hub/.style={circle, fill=gc!14, draw=gc, line width=1pt, minimum size=21pt, inner sep=0pt, font=\small},
  leaf/.style={circle, fill=hc!10, draw=hc, line width=0.7pt, minimum size=19pt, inner sep=0pt, font=\small}]
\node[hub] (s) at (0,0) {$3$};
\foreach \v/\ang in {0/150, 1/90, 2/30, 4/210, 5/330} {
  \node[leaf] (u\v) at (\ang:2.15) {$\v$};
}
\foreach \v in {0,1,2,4,5} { \draw[pairedge, line width=0.9pt] (s) -- (u\v); }
\draw[corruptedge] (u0) to[bend right=28] (u4);
\node[badge, text=gc, draw=gc!50] at ([xshift=0pt, yshift=-17pt]s) {$a_6=1$};
\foreach \v/\a/\ang in {0/4/150, 1/5/90, 2/5/30, 4/4/210, 5/5/330} {
  \node[badge] at (\ang:3.05) {$a_6=\a$};
}
\begin{scope}[xshift=4.3cm, yshift=0.55cm]
  \draw[pairedge, line width=0.9pt] (0,0.5) -- (0.7,0.5);
  \node[font=\footnotesize, anchor=west] at (0.85,0.5) {honest pair $\{3,u\}$};
  \draw[corruptedge] (0,0) -- (0.7,0);
  \node[font=\footnotesize, anchor=west] at (0.85,0) {corrupt pair $\{0,4\}$};
  \node[badge] at (0.35,-0.55) {$a_6$};
  \node[font=\footnotesize, anchor=west] at (0.85,-0.55) {absence count after six pairs};
\end{scope}
\end{tikzpicture}
\caption{The trace of \cref{ex:absence-count-trace} for the co-singleton target $h_3$. Every honest pair contains the hole $3$, so the honest pairs form a star centred at $3$; the corrupt pair avoids it. The absence count of $3$ stays at the number of corrupt pairs, and the absence count of every other point grows without bound, so \cref{alg:absence-count} returns $h_3$.}
\label{fig:co-singleton-star}
\end{figure}

\begin{algorithm}[t]
\caption{Absence count for the co-singleton class $\Hco$}
\label{alg:absence-count}
\textbf{Input:} a contrastive prefix $P_{\le n}=(p_1,\dots,p_n)$.\quad \textbf{Output:} a hypothesis in $\Hco$.
\begin{enumerate}[leftmargin=1.8em, itemsep=1pt, topsep=3pt]
  \item For each $x\in\Seen_n(P)$, compute the \emph{absence count} $a_n(x):=|\{t\le n:\ x\notin p_t\}|$.
  \item Let $s^*$ be the point of $\Seen_n(P)$ with the smallest absence count; break ties by the enumeration of $\X$.
  \item Return $h_{s^*}$.
\end{enumerate}
\end{algorithm}

\begin{theorem}[Contrastive robustness]
\label{thm:co-singleton}
$\Hco\in\kCtrId{\Fin}$, and \cref{alg:absence-count} identifies every $h_s$ from every $k$-corrupted contrastive presentation, for every $k\in\Nat$. The algorithm does not use $k$ and is computable.
\end{theorem}

\begin{proof}
Fix the target $h_s$ and let $k_0$ be the true number of corrupt pairs.
Every honest pair contains $s$, so $a_n(s)\le k_0$ for all $n$.

Let $t\ne s$. The presentation covers $\X\setminus\{s\}$. At most $k_0$ pairs are corrupt, and they contain at most $2k_0$ points, so all but finitely many points $u\ne s,t$ appear in an honest pair. This pair is $\{s,u\}$, and it misses $t$. So $a_n(t)\to\infty$.

A point that does not appear among the first $k_0+1$ pairs has an absence count of at least $k_0+1$ from then on. So only finitely many points can compete with $s$, and each of them eventually has a larger count. The algorithm returns $h_s$ from some time on.
\end{proof}

The theorem is also an instance of \cref{thm:main}: two co-singletons cover $\X$, share infinitely many points, and neither support contains the other. In fact \cref{alg:absence-count} \emph{is} the least-violation rule of that proof, because a pair crosses $h_x$ exactly when it contains $x$, so $\viol_{h_x}(P_{\le n})=a_n(x)$. The rule minimizes over the infinite class. But on $\Hco$ every unseen point has the largest possible count $n$, so the minimization runs over the finite seen set.

\begin{example}[A trace with one corrupt pair]
\label{ex:absence-count-trace}
Let $\X=\Nat$, the target be $h_3$, and
$P_{\le6}=\big(\{3,0\},\{3,1\},\underline{\{0,4\}},\{3,2\},\{3,4\},\{3,5\}\big)$, where the underlined pair is corrupt.
After six rounds the absence counts are $a_6(0)=4$, $a_6(1)=5$, $a_6(2)=5$, $a_6(3)=1$, $a_6(4)=4$, and $a_6(5)=5$, so the algorithm returns $h_3$.
On every extension with no further corrupt pair, $a_n(3)$ stays at $1$ and every other count grows without bound.
\end{example}

\subsection{Classes with disjoint holes}

The co-singleton class is the simplest member of a large family. Call $\X\setminus\supp(h)$ the \emph{hole} of $h$.

\begin{corollary}[Disjoint holes]
\label{cor:disjoint-holes}
Let $\Hyp$ be a countable class of proper nontrivial hypotheses with at least two members.
\begin{enumerate}[label=\textnormal{(\alph*)}]
  \item $\Hyp\in\kCtrId{\Fin}$ if and only if the holes of the members are pairwise disjoint and the union of any two holes misses infinitely many points.
  \item If, in addition, two members have holes with at most $r$ points, then $\Hyp\notin\kTxtId{r}$.
\end{enumerate}
\end{corollary}

\begin{proof}
(a) Two members cover $\X$ exactly when their holes are disjoint, and they share infinitely many points exactly when the union of their holes misses infinitely many points. Two nonempty disjoint holes are incomparable, so no support contains another. So the conditions are those of \cref{thm:main} for $k=\Fin$.
(b) The enumeration of $\X$ is an $r$-corrupted text for both members, because it lists each point of a hole once.
\end{proof}

\begin{example}[Co-$r$ classes]
\label{ex:co-r}
Split $\X$ into disjoint sets $S_1,S_2,\dots$ of size $r$, and let $\supp(h_i)=\X\setminus S_i$. By \cref{cor:disjoint-holes} this class lies in $\kCtrId{\Fin}$ but not in $\kTxtId{r}$. The co-singleton class is the case $r=1$. For $r\ge2$, the absence count no longer isolates the target, but the least-violation rule of \cref{thm:main} still does, and by \cref{rem:effectivity} it is computable here.
Holes need not be finite: if $S_1,S_2,\dots$ are infinite and their union misses infinitely many points, the class still lies in $\kCtrId{\Fin}$.
\end{example}

\begin{remark}[Overlap]
\label{rem:overlap}
Call the set $\supp(h)\cap\V(\GammaC(h,g))$ of positives of $h$ that some common crossing covers the \emph{overlap} of $h$ with $g$. By \cref{prop:defect-formula}, for infinite supports it is $\supp(h)\setminus\supp(g)$, the hole of $g$, if $h$ and $g$ cover $\X$, and it is infinite otherwise. So a class with at least two members, all of infinite support, in which each $h$ has a bound on its overlap with all rivals, has pairwise disjoint holes of uniformly bounded size. It is therefore a special case of \cref{cor:disjoint-holes} and of \cref{prop:progress} below.
\end{remark}

\subsection{Incomparability at every budget}

\begin{theorem}[Corrupted incomparability]
\label{thm:corrupted-incomparability}
For every $k\ge1$, the classes $\kCtrId{k}$ and $\kTxtId{k}$ are incomparable.
\end{theorem}

\begin{proof}
The co-singleton class lies in $\kCtrId{\Fin}\subseteq\kCtrId{k}$ (\cref{thm:co-singleton}). It is not in $\kTxtId{k}$, because $\kTxtId{k}\subseteq\kTxtId{1}$ and \cref{thm:text-fragile} applies.

For the other direction, fix $k$ and disjoint infinite sets $A,B\subseteq\X$ such that $\X\setminus(A\cup B)$ is not empty. Split $B$ into disjoint blocks $B_0,B_1,\dots$ of size $k+1$, and let $\Hyp_k:=\{h_i:i\ge0\}$ with $\supp(h_i):=A\cup B_i$.
A $k$-corrupted text for $h_i$ eventually shows all of $B_i$. Every point of another block is a wrong example, and a block has $k+1>k$ points, so no other block is ever shown completely. The learner that outputs $h_i$ once $B_i$ is complete therefore succeeds, and $\Hyp_k\in\kTxtId{k}$.
For $i\ne j$ the supports of $h_i$ and $h_j$ are incomparable and miss $\X\setminus(A\cup B)$, so they do not cover $\X$. By \cref{cor:clean}, $\Hyp_k\notin\Ctr\Id$, and so $\Hyp_k\notin\kCtrId{k}$.
\end{proof}

\subsection{Beyond finite budgets}
\label{subsec:beyond}

A natural next model allows infinitely many corrupt pairs, as long as they become rare. In this model honest pairs must do the covering.

\begin{definition}[Corruption at a rate]
\label{def:rate}
Let $\varphi:\Nat\to\Nat$ be nondecreasing. A \emph{$\varphi$-corrupted contrastive presentation} for $h$ is a sequence $P$ of pairs such that, for every $n$, at most $\varphi(n)$ of the pairs $p_1,\dots,p_n$ are corrupt, and the honest pairs of $P$ cover $\supp(h)$.
\end{definition}

A bounded $\varphi\equiv k$ gives the $k$-corrupted presentations whose honest pairs already cover $\supp(h)$, a smaller set than in \cref{def:presentations}. Every unbounded $\varphi$ fails completely, however slowly it grows.

\begin{proposition}[No unbounded budget]
\label{prop:unbounded}
Let $\varphi$ be nondecreasing and unbounded. Then any two distinct proper nontrivial hypotheses $h,g$ have a common $\varphi$-corrupted presentation. So under $\varphi$-corruption a class of proper nontrivial hypotheses is identifiable only if it has at most one member, for every rate from $\varphi(n)=\lfloor\varepsilon n\rfloor$ with $\varepsilon>0$, through every unbounded sublinear rate, down to rates that grow like $\log\log n$.
\end{proposition}

\begin{proof}
Fix a pair $e\in\GammaC(h,g)$, which exists by \cref{lem:trichotomy}(a), and clean presentations $(a_i)$ of $h$ and $(b_i)$ of $g$. Choose positions $t_1<t_2<\cdots$ with $\varphi(t_j)\ge j$; they exist because $\varphi$ grows without bound. Put $a_1,b_1,a_2,b_2,\dots$ at the positions $t_1,t_2,t_3,\dots$, and put $e$ at every other position.
A pair at a position other than the $t_j$ is $e$, which is honest for both hypotheses. If $t_j\le n$, then $j\le\varphi(t_j)\le\varphi(n)$, so at most $\varphi(n)$ of the first $n$ pairs lie at positions $t_j$, and only these can be corrupt. The honest pairs $a_i$ cover $\supp(h)$, and the honest pairs $b_i$ cover $\supp(g)$.
\end{proof}

The adversary in the proof spends almost all rounds on one uninformative pair. So elapsed time is not the right clock: what matters is how many positives the honest pairs have shown. On classes with small disjoint holes, honest progress that outpaces the corruption suffices, with no assumption on rates and no knowledge of either quantity.

\begin{proposition}[Progress beats corruption]
\label{prop:progress}
Let the holes $\X\setminus\supp(g)$ of the members $g$ of $\Hyp$ be pairwise disjoint, each with at most $c$ points. Let $h\in\Hyp$, let $P$ be any sequence of pairs, let $k_n:=\viol_h(P_{\le n})$, and let $m_n$ be the number of positives of $h$ that lie in some honest pair among $p_1,\dots,p_n$. If $m_n-k_n>c$ for all large $n$, in particular if $m_n-k_n\to\infty$, then the least-violation rule outputs $h$ for all large $n$.
\end{proposition}

\begin{proof}
The target has $\viol_h(P_{\le n})=k_n$. Let $g\ne h$. By \cref{lem:disjoint-charging}, $\viol_g(P_{\le n})$ is at least the number of positives of $h$ in honest pairs that also lie in $\supp(g)$. At most $c$ positives of $h$ lie in the hole of $g$, so $\viol_g(P_{\le n})\ge m_n-c>k_n$ for all large $n$.
\end{proof}

Under a finite budget $k_0$ and infinite supports, $m_n\to\infty$ and $k_n\le k_0$, so \cref{prop:progress} recovers the robustness of the co-singleton and co-$r$ classes. The proposition also allows $k_n\to\infty$, as long as honest progress stays ahead.

\section{Conclusion}
\label{sec:conclusion}

We studied identification and generation in the limit from contrastive presentations: streams of unordered pairs that disagree under the target, with the direction hidden.
One piece of combinatorics carries the theory. A pair is valid for every member of a finite family exactly when its points lie in antipodal cells of the membership map. The points without an antipodal partner are the defects that every presentation must pay for.

Read for two hypotheses, the defects give the exact budget $\lceil\kb/2\rceil$ at which they become confusable. Read for whole classes, together with a robust form of Angluin's tell-tale sets, they characterize identification at every corruption budget, from the clean case to an unknown finite budget. Read for families, they give the shared presentations that block generation.
The hidden direction has two sides. On clean data it costs: $\Ctr\Id\subsetneq\Txt\Id$. Under finite corruption this inclusion fails, because every honest pair carries a negative of the target: the classes with pairwise disjoint holes survive every finite budget with one learner, while positive data already fail for some of them, such as the co-singleton class, with a single wrong example. Corruption that grows without bound defeats every learner on every class with two or more members, so beyond finite budgets time is the wrong clock and honest progress is the right one.

Several questions remain open. \Cref{app:discussion} discusses them in more detail.
\begin{enumerate}[label=\textnormal{(\roman*)}]
  \item \emph{Ordinary generation.} On finite classes, confusable families with a finite common support are the only obstruction (\cref{thm:finite-generation}), but on countable classes they are not (\cref{thm:countable-gap}). Compatible finite witnesses suffice for every class (\cref{thm:witness}). Are they also necessary?
  \item \emph{Corrupted generation.} Robust identification lifts to robust generation (\cref{prop:fin-id-to-gen}). What is the right robust closure dimension for corrupted contrastive generation?
  \item \emph{Corruption relative to progress.} Every unbounded budget fails (\cref{prop:unbounded}), but honest progress can beat growing corruption on classes with small disjoint holes (\cref{prop:progress}). Which classes admit a learner that succeeds whenever honest progress outpaces corruption?
  \item \emph{Random pairs.} If the pairs come at random from the cut, the observed graph is a random bipartite graph. Do phase transitions govern statistical contrastive identification?
  \item \emph{Effective learners.} For finite budgets our learners use tell-tale sets as an oracle. For which classes can a learner compute them from membership or consistency oracles?
\end{enumerate}

\section*{Acknowledgments}

We thank the NeurIPS 2026 reviewers and area chair for their helpful comments.
This work was supported by the Australian Research Council (ARC) under the DECRA Fellowship DE240101049.
The authors declare no competing interests.

\paragraph{Use of large language models.}
An LLM-based assistant (Claude, Anthropic) was used to draft and check parts of the mathematics: the cell calculus of \cref{sec:crossings} and the lemmas built on it (\cref{lem:charging,lem:disjoint-charging,lem:diagonal}); the extension of the identification characterization to every corruption budget (\cref{thm:pair,thm:main}, \cref{cor:budget-hierarchy,cor:disjoint-holes}); the results on countable classes and finite witnesses (\cref{lem:coverable-core,thm:countable-gap,thm:witness}, \cref{rem:profile,rem:profile-gap}); the unbounded-budget result (\cref{prop:unbounded}); and the write-up of proofs of results that the authors stated during the review of the conference version (\cref{thm:finite-generation,prop:progress}). It was also used for editing, for the prose form of the proofs, and for the figures. All results and proofs were checked by the authors, supported by independent machine checks (brute-force verification on small ground sets and independent reconstructions of the proofs).

\bibliographystyle{plainnat}
\bibliography{ref}

\clearpage
\appendix
\crefalias{section}{appendix}\crefalias{subsection}{subappendix}\crefalias{subsubsection}{subsubappendix}
\addtocontents{toc}{\protect\addvspace{5pt}\protect\noindent{\protect\sffamily\protect\small\protect\bfseries\protect\color{muted}Appendix}\protect\par}
\noindent{\color{refcool}\rule{\linewidth}{1.6pt}}\par\vspace{8pt}
\noindent{\sffamily\bfseries\LARGE\color{navy}Appendix}\par\vspace{3pt}
\noindent{\sffamily\small\color{muted}Contrastive Identification and Generation in the Limit}\par\vspace{6pt}
\noindent{\color{rulegray}\rule{\linewidth}{0.5pt}}\par\vspace{4pt}
\suppressfloats[t]
\section{Roadmap of the Appendix}
\label{app:roadmap}

\begin{center}
\small
\begin{tabularx}{\linewidth}{@{}l>{\raggedright\arraybackslash}X@{}}
\toprule
part & contents\\
\midrule
\cref{app:proofs-id} & proofs of \cref{thm:pair} (two hypotheses) and \cref{thm:main} (identification at every budget)\\
\cref{app:proofs-gen} & proofs of \cref{thm:nonuniform-contrastive-generation} (non-uniform generation), \cref{thm:finite-generation} (generation on finite classes), \cref{thm:countable-gap} (countable classes), and \cref{thm:diamond} (diamond)\\
\cref{app:discussion} & known versus unknown budgets, counting of corruption, coverage by honest pairs, corrupted generation, generation on countable classes, random presentations, effectivity\\
\cref{app:extended-related} & additional related work\\
\bottomrule
\end{tabularx}
\end{center}

\noindent The results of \cref{sec:crossings,sec:reversal}, \cref{thm:uniform-contrastive-generation}, and the propositions of \cref{sec:generation} have complete short proofs in the main text.

\paragraph{How the proofs are organized.}
An argument that is needed more than once is stated once, as a lemma: the two charging lemmas (\cref{lem:charging,lem:disjoint-charging}), the coverable core (\cref{lem:coverable-core}), and the diagonal construction (\cref{lem:diagonal}), which drives the adversary arguments for \cref{thm:pair}(b), the necessity of (II) in \cref{thm:main}, and \cref{thm:countable-gap}. A long proof opens with a sentence that names its parts, and italic headings mark the parts. Only the displays that are cited later carry numbers.

\paragraph{Dependencies.}
\Cref{fig:dependency} shows which results each proof uses. An arrow from $X$ to $Y$ means that the proof of $Y$ cites $X$; arrows implied by other arrows are omitted.

\begin{figure}[ht]
\centering
\IfFileExists{figs/fig_dependency.pdf}{\includegraphics[width=\linewidth]{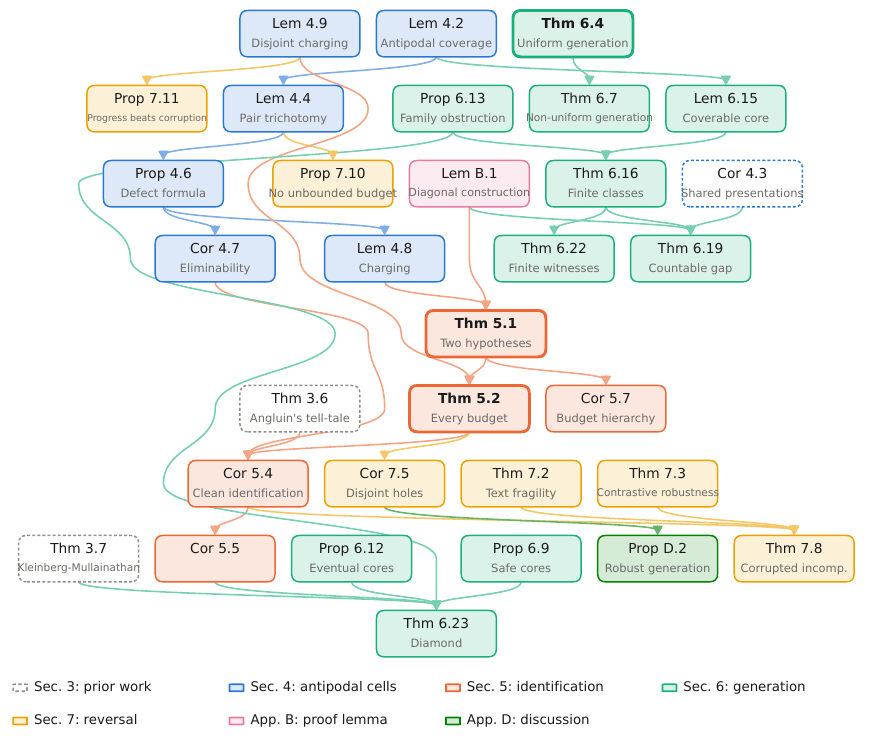}}{\fbox{dependency graph}}
\caption{Dependency graph of the results. Results from earlier work have dashed frames.}
\label{fig:dependency}
\end{figure}

\section{Proofs for Section~\ref{sec:identification}}
\label{app:proofs-id}

Three of our impossibility results, \cref{thm:pair}(b), the necessity of (II) in \cref{thm:main}, and \cref{thm:countable-gap}, use the same adversary. It builds a presentation in rounds: in each round it first follows a stream on which the learner is forced into a bad answer, and then it adds a little progress towards the presentation that it finally wants. We isolate this construction once.
Throughout, an enumeration of a finite nonempty set repeats its points, so that every enumeration is an infinite sequence.

\needspace{14\baselineskip}
\subsection{A diagonal construction}

\begin{lemma}[Diagonal construction]
\label{lem:diagonal}
Let $\mathcal S$ be a set of finite sequences of pairs, let $\mathsf E$ be a property of finite sequences of pairs, and let $Y\subseteq\X$ be nonempty. Assume that every $\sigma\in\mathcal S$ has an infinite extension $\tau_\sigma$ such that
\begin{enumerate}[label=\textnormal{(\roman*)}]
  \item every sufficiently long prefix of $\tau_\sigma$ has property $\mathsf E$, and
  \item for every prefix $\rho$ of $\tau_\sigma$ that extends $\sigma$ and every $y\in Y$, there is a finite sequence $\beta$ of pairs such that $\rho\beta\in\mathcal S$ and some pair of $\beta$ contains $y$.
\end{enumerate}
Then every $\sigma_0\in\mathcal S$ extends to an infinite sequence $P$ of pairs such that every finite prefix of $P$ is a prefix of a member of $\mathcal S$, every point of $Y$ lies in a pair of $P$, and infinitely many prefixes of $P$ have property $\mathsf E$.
\end{lemma}

\begin{proof}
Enumerate $Y$ as $y_1,y_2,\dots$. Given $\sigma_i\in\mathcal S$, use (i) to choose a prefix $\rho_i$ of $\tau_{\sigma_i}$ that is longer than $\sigma_i$ and has property $\mathsf E$, and use (ii) to choose $\beta_i$ such that $\sigma_{i+1}:=\rho_i\beta_i$ lies in $\mathcal S$ and some pair of $\beta_i$ contains $y_{i+1}$. The sequences $\sigma_0\sqsubset\rho_0\sqsubseteq\sigma_1\sqsubset\rho_1\sqsubseteq\cdots$ strictly increase, so they have an infinite limit $P$. Every finite prefix of $P$ is a prefix of some $\sigma_i\in\mathcal S$, every $y_{i+1}$ lies in a pair of $\beta_i$, and every $\rho_i$ is a prefix of $P$ with property $\mathsf E$.
\end{proof}

In each application, $\mathcal S$ consists of the streams that the adversary can still extend into a presentation of the hypothesis it finally wants, $\tau_\sigma$ is a presentation of a rival, and $\mathsf E$ says that the learner gives an answer that is wrong for the hypothesis that the adversary finally presents.

\subsection{Two hypotheses}

\begin{proof}[Proof of \cref{thm:pair}(a)]
If one sequence is a $k$-corrupted presentation of both $h$ and $g$, then an identifier produces one sequence of outputs on it, and these outputs cannot converge both to $h$ and to $g$. So (i) implies (ii).

To see that (ii) implies (iii), assume $\kb(h,g)\le 2k$; we build a shared $k$-corrupted presentation. Let $J:=\Defect(h,g)$. By \cref{prop:defect-formula}, either $J$ is empty, or $J$ is one of the cells $R_{11},R_{10},R_{01}$ and the antipodal cell of $J$ is empty. Each point $x\in J$ lies in a pair $e^h_x$ that $h$ rejects and $g$ accepts, and in a pair $e^g_x$ that $g$ rejects and $h$ accepts. Indeed, the cell patterns show that one may take
\begin{itemize}
  \item $e^h_x=\{x,b\}$ with $b\in R_{10}$ and $e^g_x=\{x,c\}$ with $c\in R_{01}$ if $J=R_{11}$;
  \item $e^h_x=\{x,a\}$ with $a\in R_{11}$ and $e^g_x=\{x,d\}$ with $d\in R_{00}$ if $J=R_{10}$;
  \item $e^h_x=\{x,d\}$ with $d\in R_{00}$ and $e^g_x=\{x,a\}$ with $a\in R_{11}$ if $J=R_{01}$.
\end{itemize}
The cells used here are not empty: \cref{lem:trichotomy}(b) gives this for $R_{10},R_{01},R_{00}$, and $R_{11}$ equals $\supp(g)$ when $R_{01}=\varnothing$ and $\supp(h)$ when $R_{10}=\varnothing$, and supports are not empty.
Now split $J$ into two disjoint sets $J_h$ and $J_g$ with at most $k$ points each. Let $P$ list the pairs $e^h_x$ for $x\in J_h$, then the pairs $e^g_x$ for $x\in J_g$, and then a sequence of common crossings that covers the rest of $\supp(h)\cup\supp(g)$. The last part exists by \cref{lem:antipodal}(b), because the points outside $J$ lie in $\V(\GammaC(h,g))$; if no such point remains, repeat a common crossing, which exists by \cref{lem:trichotomy}(a). The sequence $P$ covers both supports, $h$ rejects only the at most $k$ pairs $e^h_x$, and $g$ rejects only the at most $k$ pairs $e^g_x$.

For (iii) implies (i), recall that $\kb(h,g)=\max\{\kappa(h\to g),\kappa(g\to h)\}$ (\cref{prop:defect-formula}); by symmetry we may assume $\kappa(h\to g)\ge2k+1$. The identifier outputs $h$ as soon as $\viol_g(P_{\le n})\ge k+1$, and $g$ before that. On a $k$-corrupted presentation of $g$ the count $\viol_g$ never exceeds $k$, so the output is always $g$. On a $k$-corrupted presentation of $h$, \cref{lem:charging} gives $\viol_g(P_{\le n})\ge k+1$ for all large $n$, so the output is eventually always $h$.

Finally, by the equivalence of (ii) and (iii), the least budget with a shared presentation is the least $k$ with $\kb(h,g)\le2k$, that is, $\lceil\kb(h,g)/2\rceil$.
\end{proof}

\begin{proof}[Proof of \cref{thm:pair}(b)]
By \cref{prop:defect-formula}, both defects are nonempty only when $h$ and $g$ cover $\X$, and then both equal $\supp(h)\cap\supp(g)$. This gives the second form of the statement.

If both defects are infinite, the identifier that outputs $h$ when $\viol_h\le\viol_g$, and $g$ otherwise, succeeds without knowing the budget. On a $k_0$-corrupted presentation of $h$, the count $\viol_h$ stays at most $k_0$ while $\viol_g$ tends to infinity by \cref{lem:charging}, and the case of target $g$ is symmetric.

Conversely, let one identifier $I$ succeed on $\{h,g\}$ for every finite budget. Part~(a) at every $k$ gives $\kb(h,g)=\infty$. If the supports are not nested, \cref{prop:defect-formula} then shows that $h$ and $g$ cover $\X$ and that both defects are infinite, as claimed. It remains to exclude nested supports. By symmetry suppose $\supp(h)\subsetneq\supp(g)$, and put $C:=\supp(g)\setminus\supp(h)$. Fix $d\notin\supp(g)$, which exists because $g$ is proper. A pair $\{a,d\}$ with $a\in\supp(h)$ crosses both $h$ and $g$, and a pair $\{c,d\}$ with $c\in C$ crosses $g$ but not $h$.
Apply \cref{lem:diagonal} to the set $\mathcal S$ of all finite sequences of such pairs, with $Y:=\supp(g)$ and with property $\mathsf E$: ``$I$ outputs $h$''. For $\sigma\in\mathcal S$, let $\tau_\sigma$ be $\sigma$ followed by the pairs $\{a,d\}$ along an enumeration of $\supp(h)$. This is a presentation of $h$ whose corrupt pairs are the finitely many pairs $\{c,d\}$ in $\sigma$, so $I$ outputs $h$ on all its long prefixes. Every prefix of $\tau_\sigma$ consists of such pairs, and so does its extension by one pair $\{y,d\}$ with $y\in\supp(g)$; so condition (ii) holds as well. Starting from the empty sequence, the lemma gives a sequence of such pairs that covers $\supp(g)$: a clean presentation of $g$ on which $I$ outputs $h$ infinitely often. This contradicts the choice of $I$.
\end{proof}

\subsection{Whole classes}

\begin{proof}[Proof of \cref{thm:main}]
We prove necessity of (I), necessity of (II), and sufficiency for finite $k$ and for $k=\Fin$, in this order.

\emph{Necessity of (I).} Let $h,g\in\Hyp$ have incomparable supports. The subclass $\{h,g\}$ is identifiable at the same budget (\cref{sec:problem}). For incomparable supports, \cref{prop:defect-formula} gives $\kb(h,g)>0$ only if $h$ and $g$ cover $\X$, and then $\kb(h,g)=|\supp(h)\cap\supp(g)|$. So \cref{thm:pair} yields (I): at least $2k+1$ common points for finite $k$, and infinitely many for $k=\Fin$.

\emph{Necessity of (II) for $k=\Fin$.} By \cref{thm:pair}(b), any two members cover $\X$. Nested supports never cover $\X$, because the larger support would be $\X$. So no member lies below another, and (II) holds with $T_g=\varnothing$.

\emph{Necessity of (II) for finite $k$.} Let $I$ identify $\Hyp$ at budget $k$, and suppose that some $g\in\Hyp$ has no set $T_g$ as in (II). Then for every finite $T\subseteq\supp(g)$ some member $f$ below $g$, that is, with $\supp(f)\subsetneq\supp(g)$, misses at most $2k$ points of $T$. We first make this failure uniform: there are a set $L\subseteq\supp(g)$ with at most $2k$ points, finite sets $T_1\subseteq T_2\subseteq\cdots$ with union $\supp(g)$ and $L\subseteq T_1$, and members $f_1,f_2,\dots$ below $g$ with
\begin{equation}
\label{eq:uniform-failure}
   T_r\setminus\supp(f_r)=L\qquad\text{for every }r.
\end{equation}
To find them, let $U_1\subseteq U_2\subseteq\cdots$ be finite sets with union $\supp(g)$, and choose for each $s$ a member $f'_s$ below $g$ for which $\beta_s:=U_s\setminus\supp(f'_s)$ has at most $2k$ points. For fixed $j$, the sets $\beta_s\cap U_j$ with $s\ge j$ are subsets of $U_j$ with at most $2k$ points, so they take finitely many values. Repeated use of the pigeonhole principle gives infinite sets $S_1\supseteq S_2\supseteq\cdots$ and sets $L_j$ with $\beta_s\cap U_j=L_j$ for all $s\in S_j$ with $s\ge j$. Then $L_i=L_j\cap U_i$ for $i\le j$, so $L:=\bigcup_jL_j$ has at most $2k$ points and $L\cap U_j=L_j$. Choose $j_0$ with $L\subseteq U_{j_0}$ and, for each $r$, an index $s_r\in S_{j_0+r}$ with $s_r\ge j_0+r$. The sets $T_r:=U_{j_0+r}$ and the members $f_r:=f'_{s_r}$ satisfy \eqref{eq:uniform-failure}, because $T_r\subseteq U_{s_r}$ gives $T_r\setminus\supp(f_r)=\beta_{s_r}\cap T_r=L_{j_0+r}=L$.

Now fix $d\notin\supp(g)$, which exists because $g$ is proper, and a point $p_0\in\supp(g)\setminus L$, which exists because this set contains the nonempty set $\supp(f_1)$. Dropping finitely many indices $r$, we may assume $p_0\in T_1$, so that $p_0\in\supp(f_r)$ for all $r$ by \eqref{eq:uniform-failure}. Split $L$ into two sets $L',L''$ with at most $k$ points each, and let $\sigma_0$ list the pairs $\{x,d\}$ for $x\in L'$ and $\{x,p_0\}$ for $x\in L''$. By \eqref{eq:uniform-failure}, every $x\in L$ has $g(x)=1$ and $f_r(x)=0$; moreover $g(d)=f_r(d)=0$ and $g(p_0)=f_r(p_0)=1$. So $g$ rejects only the pairs $\{x,p_0\}$ of $\sigma_0$, and each $f_r$ rejects only the pairs $\{x,d\}$; in both cases at most $k$ pairs.

Let $\mathcal S$ be the set of sequences that consist of $\sigma_0$ followed by finitely many pairs $\{q,d\}$ with $q\in\supp(g)\setminus L$. Apply \cref{lem:diagonal} to $\mathcal S$, with $Y:=\supp(g)\setminus L$ and with property $\mathsf E$: ``$I$ outputs a hypothesis other than $g$''. For $\sigma\in\mathcal S$, choose $r$ so large that $T_r$ contains every point $q$ of a pair $\{q,d\}$ that $\sigma$ adds to $\sigma_0$; these points then lie in $\supp(f_r)$ by \eqref{eq:uniform-failure}. Let $\tau_\sigma$ be $\sigma$ followed by the pairs $\{a,d\}$ along an enumeration of $\supp(f_r)$. Its only pairs outside $\Cut(f_r)$ lie in $\sigma_0$, and it covers $\supp(f_r)$, so it is a $k$-corrupted presentation of $f_r\ne g$, and (i) holds. Every prefix of $\tau_\sigma$ that extends $\sigma$ lies in $\mathcal S$, because $\supp(f_r)\subseteq\supp(g)\setminus L$; appending the pair $\{y,d\}$ gives (ii). The lemma yields a sequence whose pairs outside $\Cut(g)$ all lie in $\sigma_0$ and which covers $L$ and $\supp(g)\setminus L$. It is a $k$-corrupted presentation of $g$ on which $I$ outputs another hypothesis infinitely often, a contradiction.

\emph{Sufficiency for finite $k$.} Enumerate $\Hyp=\{h_0,h_1,\dots\}$, and call $h_i$ \emph{eligible} at time $n$ if $T_{h_i}\subseteq\Seen_n(P)$ and $\viol_{h_i}(P_{\le n})\le k$. The identifier outputs the eligible hypothesis of least index, and $h_0$ if none is eligible. Fix a target $h=h_i$ and a $k$-corrupted presentation $P$ of $h$. The target is eligible from some time on: its count $\viol_h$ never exceeds $k$, and the finite set $T_h\subseteq\supp(h)$ eventually appears. It therefore suffices to show that each of the finitely many $h_j$ with $j<i$ is eventually never eligible. Because $h_j\ne h$, there are two cases.

If $\supp(h)\not\subseteq\supp(h_j)$, then $\kappa(h\to h_j)\ge 2k+1$. Indeed, if $\supp(h_j)\subsetneq\supp(h)$, then $\Defect(h\to h_j)=\supp(h)\setminus\supp(h_j)$ by \cref{prop:defect-formula}, and this set contains at least $2k+1$ points of $T_h$ by (II). Otherwise the supports are incomparable, and by (I) and \cref{prop:defect-formula} the defect is $\supp(h)\cap\supp(h_j)$, with at least $2k+1$ points. Now \cref{lem:charging} gives $\viol_{h_j}(P_{\le n})\ge k+1$ for all large $n$, and $h_j$ is not eligible from then on.

If $\supp(h)\subsetneq\supp(h_j)$, then $h_j$ is never eligible. Let $C:=\supp(h_j)\setminus\supp(h)$. By (II) for $h_j$ and $f=h$, at least $2k+1$ points of $T_{h_j}$ lie in $C$. So at a time at which $h_j$ is eligible, the stream has shown at least $2k+1$ points of $C$. But an honest pair contains exactly one negative of $h$, so at most one point of $C$. If it contains one, its other endpoint is a positive of $h$ and so of $h_j$, and $h_j$ rejects the pair. A corrupt pair contains at most two points of $C$, and if it contains two, then $h_j$ rejects it too. Let $a$ be the number of honest pairs in $P_{\le n}$ that meet $C$, and let $b_1$ and $b_2$ be the numbers of corrupt pairs with exactly one and with two points in $C$. The number of points of $C$ in $\Seen_n(P)$ is then at most
\begin{equation}
\label{eq:superset-count}
   a+b_1+2b_2=(a+b_2)+(b_1+b_2)\ \le\ \viol_{h_j}(P_{\le n})+k ,
\end{equation}
which is at most $2k$ at any time at which $h_j$ is eligible, a contradiction.

\emph{Sufficiency for $k=\Fin$.} Now the threshold in (II) is infinite, and a finite set $T_g$ cannot miss infinitely many points. So (II) says that no support lies below another, and all supports are pairwise incomparable. By (I), any two members then cover $\X$ and share infinitely many points. In particular the holes $\X\setminus\supp(g)$, $g\in\Hyp$, are pairwise disjoint. The identifier outputs the least-index hypothesis with the fewest violations, which exists because violation counts are natural numbers. Fix a target $h$, a budget $k_0$, and a $k_0$-corrupted presentation $P$ of $h$. Let $P_n$ be the set of positives of $h$ that lie in an honest pair among $p_1,\dots,p_n$. By \cref{lem:disjoint-charging},
\begin{equation}
\label{eq:fin-charging}
   \viol_g(P_{\le n})\ \ge\ |P_n\cap\supp(g)|\qquad\text{for every }g\ne h.
\end{equation}
Every positive of $h$ lies in a pair of $P$. A positive that lies in no honest pair lies in one of the at most $k_0$ corrupt pairs, so there are at most $2k_0$ such positives. Thus the sets $P_n$ exhaust $\supp(h)$ up to at most $2k_0$ points, and $|P_n\cap\supp(g)|\to\infty$ for every $g\ne h$, because $\supp(h)\cap\supp(g)$ is infinite. In particular $|P_n|\to\infty$.
Call a rival $g$ \emph{lagging} at time $n$ if $|P_n\cap\supp(g)|\le k_0$. A lagging rival has at least $|P_n|-k_0$ points of $P_n$ in its hole, and the holes are disjoint, so two lagging rivals would give $|P_n|\ge2(|P_n|-k_0)$. Thus at most one rival lags once $|P_n|\ge2k_0+1$. A rival that stops lagging never lags again, because $P_n$ only grows; so from some time on at most one fixed rival can lag, and eventually it stops lagging as well. From then on every rival has more than $k_0\ge\viol_h(P_{\le n})$ violations by \eqref{eq:fin-charging}, and the identifier outputs $h$. It never uses $k_0$.
\end{proof}

\section{Proofs for Section~\ref{sec:generation}}
\label{app:proofs-gen}

\begin{proof}[Proof of \cref{thm:nonuniform-contrastive-generation}]
Suppose first that $G$ is a non-uniform contrastive generator for $\Hyp$, and for $m\ge1$ let $\Hyp_m$ be the set of those $h\in\Hyp$ for which the number $m$ can serve as $d_h$. These sets form a chain whose union is $\Hyp$, and $G$ is a uniform generator for $\Hyp_m$ with sample complexity $m$. So $\CDelta(\Hyp_m)\le m-1$ by the necessity part of \cref{thm:uniform-contrastive-generation}.

Conversely, let $\Hyp=\bigcup_m\Hyp_m$ for a chain $\Hyp_1\subseteq\Hyp_2\subseteq\cdots$ with $\CDelta(\Hyp_m)<\infty$, and put $d_m:=m+\CDelta(\Hyp_m)+1$. The generator works as follows on a prefix with $E:=E_n(P)$. It chooses the largest $m$ with $d_m\le|E|$; there are only finitely many candidates, because $d_m>m$. If some member $g$ of $\Hyp_m$ satisfies $E\subseteq\Cut(g)$, it outputs the least point of $\langle E\rangle^{\Cut}_{\Hyp_m}\setminus\V(E)$. This set is not empty, because $|E|\ge d_m>\CDelta(\Hyp_m)$ and so $E$ is not hollow for $\Hyp_m$. In all other cases it outputs an arbitrary point.

Let $h\in\Hyp_{m_0}$; we show that $d_h:=d_{m_0}$ works for $h$. Let $P$ be a sequence of pairs in $\Cut(h)$ with $|E_n(P)|\ge d_{m_0}$. Then $m_0$ is a candidate, so the chosen index $m$ is at least $m_0$, and $h\in\Hyp_m$ because the sets increase. Because $E\subseteq\Cut(h)$, the closure $\langle E\rangle^{\Cut}_{\Hyp_m}$ lies in $\supp(h)$, and the output is a point of $\supp(h)$ outside $\V(E)=\Seen_n(P)$.
\end{proof}

\begin{proof}[Proof of \cref{thm:finite-generation}]
\emph{Necessity.} For a confusable family with at least two members, this is \cref{prop:finite-family-obstruction}. A single hypothesis $h$ needs an infinite support, because a presentation of $h$ eventually shows all of a finite support and leaves no new positive to output.

\emph{Sufficiency.} Let $K_n$ be the set of hypotheses with no violation on $P_{\le n}$. The generator outputs the least unseen point of the common support of the coverable core $D(K_n)$, and an arbitrary point if $K_n$ or $D(K_n)$ is empty. By \cref{lem:coverable-core}, a nonempty core is confusable, so its common support is infinite by assumption, and the output exists.

Fix a target $h$ and a clean presentation $P$ of $h$. The sets $K_n$ decrease and $\Hyp$ is finite, so from some time $n_0$ on they equal the set $K$ of hypotheses $f$ with every pair of $P$ in $\Cut(f)$. The target lies in $K$, because $P$ is clean. Every pair of $P$ crosses every cut of $K$, so $P$ is a sequence of pairs of $\GammaC(K)$, and because it covers $\supp(h)$, we get $\supp(h)\subseteq\V(\GammaC(K))$. So $h\in D(K)$, and for $n\ge n_0$ the output is an unseen point of $\bigcap_{g\in D(K)}\supp(g)\subseteq\supp(h)$.
\end{proof}

\begin{proof}[Proof of \cref{thm:countable-gap}]
We first show that every confusable family $\F\subseteq\Hyp$ has an infinite common support; by \cref{thm:finite-generation}, every finite subclass then lies in $\Ctr\Gen$. If $a\notin\F$, every member of $\F$ contains $B_1$; if $b\notin\F$, every member contains $A_1$; and if $\F=\{a,b\}$, the common support is $C$. In the remaining case $\F=\{a,b\}\cup S$ for a nonempty set $S$ of hypotheses $s_m$, and $\F$ is not confusable. Indeed, take $c_k$ with $k$ larger than every index in $S$. Its pattern is $1$ on $a$ and $b$ and $0$ on $S$. A point with the antipodal pattern would lie in every member of $S$ but outside $\supp(a)\cup\supp(b)$, and there is no such point, because every $s_m$ lies inside $\supp(a)\cup\supp(b)$. So \cref{cor:shared} applies.

Now let $G$ be a contrastive generator for $\Hyp$; we derive a contradiction. Fix $z\in N$, and call a pair \emph{admissible} if it has one of the forms
\[
   \{c_k,z\},\qquad \{\alpha,\beta'\}\ \ (\alpha\in A_1,\ \beta'\in B_2),\qquad \{\beta,\alpha'\}\ \ (\beta\in B_1,\ \alpha'\in A_2).
\]
Reading off the supports, every admissible pair crosses the cuts of $a$ and $b$. The last two forms also cross the cut of every $s_m$, and $\{c_k,z\}$ crosses the cut of $s_m$ exactly when $k\le m$.
Apply \cref{lem:diagonal} to the set $\mathcal S$ of all finite sequences of admissible pairs, with $Y:=\supp(a)\cup\supp(b)$ and with property $\mathsf E$: ``$G$ outputs a point of $A_1\cup B_1$''. Let $\sigma\in\mathcal S$, and choose $m\ge1$ with $k\le m$ for every pair $\{c_k,z\}$ in $\sigma$. Let $\tau_\sigma$ be $\sigma$ followed by admissible pairs that cover $\{c_1,\dots,c_m\}\cup A_1\cup B_1$, using the form $\{c_k,z\}$ only for $k\le m$. Every prefix of $\tau_\sigma$ lies in $\mathcal S$, and every point of $Y$ lies in an admissible pair; so condition (ii) holds. For (i), note that $\tau_\sigma$ is a clean presentation of $s_m$. So $G$ eventually outputs only unseen points of $\supp(s_m)$, and once $c_1,\dots,c_m$ have appeared these lie in $A_1\cup B_1$.

Starting from the empty sequence, the lemma gives a sequence $P$ of admissible pairs that covers $\supp(a)\cup\supp(b)$. It is a clean presentation of both $a$ and $b$, and $G$ outputs a point of $A_1\cup B_1$ on infinitely many of its prefixes. Then $G$ outputs infinitely often a point of $A_1$, which misses $\supp(b)$, or infinitely often a point of $B_1$, which misses $\supp(a)$. In either case $G$ fails on $P$.
\end{proof}

\begin{proof}[Proof of \cref{thm:diamond}]
All classes below are countable UUS classes of proper nontrivial hypotheses. We prove the four inclusions first and then the separations.

\emph{Inclusions.} Given an identifier $I$, let a generator output the least unseen point of $\supp(I(P_{\le n}))$, which exists because supports are infinite. Once $I$ has converged to the target, every output is a new positive of the target; so $\Ctr\Id\subseteq\Ctr\Gen$. The inclusions $\Ctr\Gen\subseteq\Txt\Gen$ and $\Txt\Id\subseteq\Txt\Gen$ hold because $\Txt\Gen$ contains every countable UUS class (\cref{thm:KM}), and $\Ctr\Id\subseteq\Txt\Id$ is \cref{cor:ctrid-txtid}.

\emph{Separations.} Three classes suffice.
\begin{itemize}
  \item The augmented class of \cref{ex:augmented} lies in $\Ctr\Gen\setminus\Ctr\Id$. Every member contains $A$, so every safe set contains $A$, and \cref{prop:safe-sufficiency} applies. Two members $h_m,h_{m'}$ with $m\ne m'$ have incomparable supports that do not cover $\X$, so \cref{cor:clean} excludes $\Ctr\Id$.
  \item The disjoint pair $\{h_A,h_B\}$ lies in $\Txt\Id\setminus\Ctr\Gen$. From texts, the first example names the target. The stream $\{a_1,b_1\},\{a_2,b_2\},\dots$ is a shared contrastive presentation, so the pair is confusable, and its common support $A\cap B$ is empty; \cref{prop:finite-family-obstruction} excludes $\Ctr\Gen$.
  \item The punctured class of \cref{ex:punctured} lies in $\Ctr\Gen\setminus\Txt\Id$. The sequence $(a_m)$ is an eventual core, so \cref{prop:eventual-core} applies. Every finite $T\subseteq A$ misses some $a_m$, and then $T\subseteq\supp(h_m)\subsetneq\supp(h_\infty)$; so $h_\infty$ has no tell-tale set, and \cref{thm:angluin} excludes $\Txt\Id$.
\end{itemize}
The first class separates $\Ctr\Id$ from $\Ctr\Gen$. The second separates $\Ctr\Gen$ from $\Txt\Gen$ and $\Ctr\Id$ from $\Txt\Id$, and shows $\Txt\Id\not\subseteq\Ctr\Gen$. The third separates $\Txt\Id$ from $\Txt\Gen$ and shows $\Ctr\Gen\not\subseteq\Txt\Id$.
\end{proof}

\section{Further Discussion}
\label{app:discussion}

\subsection{Known and unknown budgets}
\label{app:budgets}

Two obstacles separate a pairwise threshold from an identifier for a whole class.
First, a corrupted presentation of $h$ can be invalid for a rival $g$ only because it never covers part of $\supp(g)$. A learner cannot detect such undercoverage in finite time.
Second, a defect forces violations eventually, but the adversary chooses when they occur.
\Cref{thm:main} shows how each budget regime handles them.
\begin{itemize}
  \item \emph{Known finite budget.} The learner uses the fixed threshold $k$, so the timing of violations does not matter. Undercoverage matters only for rivals above the target, and the robust tell-tale sets make it visible: such a rival needs $2k+1$ shown points outside the target's support, and the adversary can show at most $2k$ of them (see \eqref{eq:superset-count} in the proof of \cref{thm:main}). \Cref{ex:pairwise-not-enough} shows that the pairwise threshold alone is not enough.
  \item \emph{Unknown budget.} Undercoverage is fatal for nested supports: the adversary presents the honest data of the larger hypothesis as corruption of the smaller one (\cref{thm:pair}(b)). Without nested supports, the holes are pairwise disjoint, and the timing obstacle disappears, because once enough positives have appeared in honest pairs, at most one rival can lag behind (see the sufficiency part for $k=\Fin$ in the proof of \cref{thm:main}).
  \item \emph{One-sided defects.} If $\kappa(h\to g)=\infty$ and $\kappa(g\to h)<\infty$, then $\kappa(g\to h)=0$ by \cref{prop:defect-formula}, and the supports are nested with an infinite difference. Such a pair is identifiable at every known budget but not at an unknown one.
\end{itemize}

\paragraph{Counting corruption.}
We count corrupt pairs over indices, with repetitions (\cref{def:presentations}). If one counted only \emph{distinct} corrupt pairs, an adversary could repeat one corrupt pair forever. On the co-singleton class with target $h_0$, for example, the stream $\{1,2\},\{0,1\},\{1,2\},\{0,2\},\{1,2\},\{0,3\},\dots$ has one distinct corrupt pair, but $a_n(1)<a_n(0)$ for all $n\ge2$, so \cref{alg:absence-count} fails. The remedy is to count distinct violating pairs in the learners. The counting arguments of \cref{lem:charging,lem:disjoint-charging} and \cref{app:proofs-id} already count distinct pairs, so \cref{thm:main} holds unchanged under this reading.

\paragraph{Coverage by honest pairs.}
In \cref{def:presentations} corrupt pairs may help to cover the positives. If only honest pairs may cover, as in \cref{def:rate}, the least shared budget rises from $\lceil\kb/2\rceil$ to $\kb$: $h$ and $g$ then share a presentation with at most $k$ corrupt pairs for each of them if and only if $\kb(h,g)\le k$.
Indeed, every point of $\Defect(h\to g)$ needs its own honest pair for $h$, and $g$ rejects that pair; so $g$ sees at least $\kappa(h\to g)$ violations, and symmetrically. Conversely, take one pair of $\Cut(h)\setminus\Cut(g)$ for each point of $\Defect(h\to g)$, one pair of $\Cut(g)\setminus\Cut(h)$ for each point of $\Defect(g\to h)$, and common crossings for all other points, and repeat only the common crossings. This gives a shared presentation with $\kappa(g\to h)$ violations for $h$ and $\kappa(h\to g)$ for $g$.
So the factor $\tfrac12$ in \cref{thm:pair} measures exactly what corrupt pairs gain when they may cover positives.

\subsection{Corrupted contrastive generation}
\label{app:corrupted-gen}

\begin{definition}[Corrupted contrastive generation]
\label{def:k-corrupted-gen}
For $k\in\Nat$, a generator $G$ is a \emph{$k$-corrupted contrastive generator} for $\Hyp$ if for every $h\in\Hyp$ and every $k$-corrupted contrastive presentation $P$ of $h$ there is $N$ with $G(P_{\le n})\in\supp(h)\setminus\Seen_n(P)$ for all $n\ge N$. We write $\Hyp\in k\text{-}\Ctr\Gen$, and $\Hyp\in\Fin\text{-}\Ctr\Gen$ if one generator works for every finite budget.
\end{definition}

The closure-based generator of \cref{thm:uniform-contrastive-generation} does not survive corruption. One corrupt pair $p^*\notin\Cut(h)$ removes the target $h$ from the version space $\HDelta(E)$. The remaining hypotheses need not have supports inside $\supp(h)$, so the closure can contain negatives of $h$. The identification results of \cref{sec:identification} use counts that tolerate finitely many inserted pairs; the closure operator has no such slack. Identification nevertheless lifts to generation:

\begin{proposition}[Robust identification gives robust generation]
\label{prop:fin-id-to-gen}
On countable UUS classes, $\kCtrId{k}\subseteq k\text{-}\Ctr\Gen$ for every $k\in\Nat$, and $\kCtrId{\Fin}\subseteq\Fin\text{-}\Ctr\Gen$. In particular every class with disjoint holes as in \cref{cor:disjoint-holes} lies in $\Fin\text{-}\Ctr\Gen$.
\end{proposition}

\begin{proof}
Let $I$ be the identifier and let $G$ output the least point of $\supp(I(P_{\le n}))$ that is neither seen nor output before. Once $I$ has converged to the target $h$, the support $\supp(h)$ is infinite and the excluded set is finite, so every output is a new positive of $h$. The generator uses the budget only through $I$. The last claim follows from \cref{cor:disjoint-holes}.
\end{proof}

The converse direction is open: can a class in $\Ctr\Gen\setminus\Ctr\Id$ lie in $\Fin\text{-}\Ctr\Gen$? A natural tool would be a \emph{robust closure dimension} that measures how far $k$ inserted pairs can push the closure outside the target's support.

\subsection{Ordinary generation on countable classes}
\label{app:ordinary-gen}

The proof of \cref{thm:finite-generation} uses finiteness once: the set $K_n$ of consistent hypotheses stops changing, and from then on the target lies in the coverable core $D(K_n)$. \Cref{thm:countable-gap} shows that this step cannot be repaired by any condition on finite families. In its class, the consistent set keeps changing, and at every phase the active family $\{a,b,s_m\}$ hides the targets $a$ and $b$ behind a finite core.

What remains open is an exact characterization. Compatible witnesses suffice for every class (\cref{thm:witness}), and they are exact on finite classes. For generation from text, \citet{lhjg26b} prove the matching necessity: they turn any generator into one that depends only on the set of observed examples, by a search through histories whose outputs are still unconfirmed. Their argument uses one feature of texts: an output that has appeared in the text is a valid example, so the search may skip it without losing a genuine error. In a contrastive presentation negatives appear too. An output that has appeared may therefore be an error, and the same search can skip the very errors that it needs to track. We do not know whether every class in $\Ctr\Gen$ has a compatible witness assignment, or whether the order of the observed pairs can matter for contrastive generation.

\subsection{Random contrastive presentations}
\label{app:statistical}

Our presentations are adversarial. A natural statistical model draws the pairs independently from a distribution $\mu$ on $[\X]^2$ whose support lies in $\Cut(h)$; coverage of the positives then holds almost surely if and only if every positive lies in some pair of the support of $\mu$.
The observed graph is then a random subgraph of the complete bipartite graph between $\supp(h)$ and its complement. For natural choices of $\mu$ it resembles a bipartite Erd\H{o}s--R\'enyi graph, and recovering the cut resembles recovering a planted bipartition in a stochastic block model. We expect phase transitions: a critical edge density below which contrastive identification fails, and above which spectral or message-passing methods succeed. \Cref{rem:hidden-bit} suggests the relevant quantity: the learner must resolve one hidden bit per component of the observed graph.

\subsection{Effective procedures}
\label{app:effective}

Our characterizations are information-theoretic.
\begin{itemize}
  \item The learner of \cref{thm:main} for finite $k$ uses the robust tell-tale sets, which, as in Angluin's theorem, exist but need not be computable.
  \item The least-violation learner for $k=\Fin$ needs only membership: by \cref{rem:effectivity} its truncation to the first $n$ hypotheses is computable for indexed families with decidable membership. On the co-singleton class it is \cref{alg:absence-count}.
  \item The closure-based generator of \cref{thm:uniform-contrastive-generation} needs an oracle that decides whether $\HDelta(E)\ne\varnothing$ and an oracle for membership in the closure $\langle E\rangle^{\Cut}_\Hyp$. The non-uniform generator also needs the chain and its dimensions.
\end{itemize}
A natural next step is to find conditions under which the robust tell-tale sets can be computed from a consistency oracle.

\section{Additional Related Work}
\label{app:extended-related}

\subsection{Identification in the limit}

The paradigm starts with \citet{gold67}, and the tell-tale theorem of \citet{angluin80} characterizes identification from positive data.
Earlier, \citet{angluin79} introduced pattern languages, a concrete class that is identifiable from positive data despite the negative results of \citet{gold67} for broader classes; this work anticipates much of the structure behind tell-tale conditions.
\citet{wharton1974approximate} studies an \emph{approximate} variant in which the learner may deviate slightly from the target language, an early relative of noise- and corruption-tolerant limit learning.
\citet{lzz08} give a detailed survey.

Recent work relaxes the success criterion.
\citet{cpt25} let the learner output a short list of candidates and succeed if the target is on the list, and they characterize list identifiability.
\citet{pf25} characterize the limit learnability of recursive functions when the learner sees evaluations on every domain point.
\citet{psv26} add computational traces of the accepting machine to Gold's model and obtain identifiability across the Chomsky hierarchy with different levels of corruption tolerance.

\subsection{Generation in the limit}

\citet{km24} introduced generation in the limit and proved that it is possible for every countable class with infinite supports.
\citet{lrt25} restated the model in learning-theoretic notation and introduced the closure dimension, which characterizes uniform generation; it is the positive-data ancestor of our contrastive closure dimension.

One thread refines the success criterion.
\citet{kmv24a,kmv24b} characterize generation under breadth constraints and study the trade-off between hallucination and mode collapse.
\citet{prr25} introduce representative generation, which asks the generator to cover meaningful subgroups of the target.
\citet{cp24,cp25b} study further aspects of generation in the limit and Pareto-optimal trade-offs in non-uniform generation.
\citet{kw25,kw25a} introduce density measures and partial-enumeration variants.
\citet{abck25} prove complexity barriers between generation modes, and \citet{kw26} study Banach density for strings embedded in a $d$-dimensional space.

A second thread studies noise.
\citet{rr25} analyze generation from noisy examples, \citet{bpz25} study noise, loss, and feedback, and \citet{mvyz25} allow infinite contamination.
\citet{lz26} propose quantitative measures of noise, and \citet{rvs26} study a replay model that captures forms of model collapse.

A third thread studies structural extensions.
\citet{hkmv25} show that generatability is not closed under countable unions, \citet{kmsv25} study the (im)possibility of automated hallucination detection, \citet{lrt26} extend generation to metric spaces, \citet{hp26} introduce agnostic identification and generation, and \citet{akk26} formalize safe generation.
None of these works studies pairs with a hidden direction.

\subsection{Learning from pair-level supervision}

Pairwise supervision is a classical theme in clustering and weakly supervised learning.
Constrained clustering with must-link and cannot-link constraints goes back to \citet{wcrs01}.
Learning from pairwise similarity and unlabeled data \citep{bns18}, learning from noisy similar and dissimilar pairs \citep{dbs21}, learning from positive and unlabeled data \citep{en08,bd20}, complementary-label learning \citep{inhs17}, learning from label proportions \citep{qsmcl09}, and multiple-instance learning with exactly one positive \citep{ccgg18} all learn from set- or pair-level signals instead of per-example labels.
These are finite-sample statistical theories: concentration of measure decides feasibility, and the analysis tolerates independent label noise.
We keep the pair-level signal but move to the asymptotic limit regime, where combinatorial structure replaces concentration and an adversary, not a noise distribution, corrupts the stream.

Closest in name is the recent learning-theoretic line on \emph{contrastive examples}. \citet{mss+25} formalize active learning with a contrastive-example oracle and connect it to self-directed learning, and \citet{rrm+26} study half-spaces under perturbed contrastive examples. The older comparison-query model \citep{klmz17} is a further relative.
All three learn from two-point signals, as we do, but with two differences: the learner knows which point carries which label, and success is measured by sample or query complexity.
We study the same two-point disagreement with the orientation hidden, and we measure feasibility in the limit. The hidden orientation is exactly what separates the two settings (\cref{rem:not-collapse}).

\end{document}